\documentclass[letterpaper]{article} 
\usepackage{aaai25}  
\usepackage{times}  
\usepackage{helvet}  
\usepackage{courier}  
\usepackage[hyphens]{url}  
\usepackage{graphicx} 
\usepackage{natbib}  
\usepackage{caption} 
\usepackage{listings}

\usepackage{booktabs}
\usepackage[all,cmtip]{xy}
\usepackage{amsmath, amssymb, amsthm}
\usepackage{algorithm, algpseudocode}
\usepackage{dsfont}
\usepackage[shortlabels]{enumitem}
\usepackage{mathrsfs}
\usepackage{multirow}
\usepackage{siunitx}
\usepackage{xcolor}

\newtheorem{cor}{Corollary}
\newtheorem{defn}{Definition}
\newtheorem{lemma}{Lemma}

\newtheorem{thm}{Theorem}

\DeclareMathSymbol{\shortminus}{\mathbin}{AMSa}{"39}

\newcommand\blank[1]{\rule[-.2ex]{#1}{.4pt}}
\def\removal{y\ominus\hat y_\shortminus}
\def\removalC{y\ominus C}
\def\addition{\hat y\ominus\hat y_+}
\def\additionC{\hat y\ominus C}
\definecolor{officegreen}{rgb}{0.0, 0.5, 0.0}

\DeclareCaptionStyle{ruled}{labelfont=normalfont,labelsep=colon,strut=off} 
\floatstyle{ruled}
\newfloat{listing}{tb}{lst}{}
\floatname{listing}{Listing}

\title{Efficient Connectivity-Preserving Instance Segmentation \\ with Supervoxel-Based Loss Function}
\author{
    Anna Grim, Jayaram Chandrashekar, Uygar S\"umb\"ul
}
\affiliations{
    Allen Institute\\

    615 Westlake Avenue \\
    Seattle, WA 98109 USA\\
    anna.grim@alleninstitute.org
}

\begin{document}

\maketitle

\begin{abstract}
  Reconstructing the intricate local morphology of neurons and their long-range projecting axons can address many connectivity related questions in neuroscience. The main bottleneck in connectomics pipelines is correcting topological errors, as multiple entangled neuronal arbors is a challenging instance segmentation problem. More broadly, segmentation of curvilinear, filamentous structures continues to pose significant challenges. To address this problem, we extend the notion of simple points from digital topology to connected sets of voxels (i.e. supervoxels) and propose a topology-aware neural network segmentation method with minimal computational overhead. We demonstrate its effectiveness on a new public dataset of 3-d light microscopy images of mouse brains, along with the benchmark datasets DRIVE, ISBI12, and CrackTree.
\end{abstract}

\section{Introduction}

High-throughput neuron reconstruction is a challenging 3-d instance segmentation problem and a major bottleneck of many data-driven neuroscience studies~\citep{gouwens2020integrated, winnubst2019reconstruction}. Deep learning-based methods are the leading framework for segmenting individual neurons, which is a critical step in reconstructing neural circuits~\citep{turaga2010, jan2018, lee2019, schmidt2024}. While these methods and many others have significantly improved segmentation quality, automated reconstruction of neurons from large 3-d images still suffers from topological errors (i.e. splits and merges) and requires substantial human proofreading. This can be attributed to two basic observations: (i) neuronal branches can be as thin as a single voxel and (ii) branches often touch or even overlap due to the imaging resolution. Consequently, a seemingly innocuous mistake at the single voxel level can produce catastrophically incorrect segmentations. 

A natural solution to these challenges is to account for the topology of the underlying objects during training. Several methods use persistent homology to capture the topology of segmented objects in terms of their Betti numbers~\citep{clough2019, hu2021topology, shit2023}. While these methods have achieved state-of-the-art results, they are computationally expensive with complexities that scale nonlinearly. In this paper, we present a more efficient approach based on an extension of simple voxels from digital topology. Our method aims to preserve the connectivity of segmented objects and has significantly faster training times that scale linearly.

In digital topology, a simple point in a 3-d binary image is defined as a foreground voxel whose deletion does not change the topology of the image~\citep{kong1989digital}. Specifically, deleting a simple voxel does not result in splits or merges, nor does it create or eliminate loops, holes, or objects. This concept has been used to penalize errors involving non-simple voxels when segmenting neurons~\citep{gornet2019reconstructing}. However, topological errors often involve multiple connected voxels, which is not addressed by a single voxel-based approach.

To overcome this limitation, we extend the concept of simple voxels to supervoxels (i.e., connected sets of voxels). We then propose a differentiable loss function based on this supervoxel characterization, which enables neural networks to be trained to minimize split and merge errors. Finally, we evaluate our approach on 3-d light microscopy images of mouse brains as well as on several 2-d segmentation datasets.

\section{Related Works}
Accurately segmenting fine-scale structures such as neurons, vessels, and roads from satellite images is a complex and well-studied problem. There are two general approaches: (1) improve feature representations~(e.g.,~\citep{wu2017, mosinka2018, hu2023learn, sheridan2023}) and (2) incorporate topology-inspired loss functions during training. The latter approach focuses on identifying critical locations where the neural network is prone to topological errors, and then using gradient-based updates to guide improvements.

\citet{bertrand1994new} studied topological characterization of simple voxels. The remarkable success of this work is that a complete topological description of a voxel can be derived through basic operations on its immediate neighborhood. Notably, this work does not invoke modern topological concepts such as homology or Morse theory to characterize voxels in low-dimensional digital topology. In many ways, the present study is inspired by that work. 

The method most closely related to ours is centerlineDice (clDice), which is also a connectivity-aware loss function~\citep{shit2021}. In this approach, soft skeletons of both the prediction and ground truth are generated by applying min- and max-pooling $k$ times during training. The loss is then computed by evaluating the overlap between the segmented objects and skeletons. Similar to our method, clDice also preserves topology up to homotopy equivalence -- if the hyperparameter $k$ is greater than the maximum observed radius.

Recently, persistent homology has been used in deep learning frameworks to track higher-order topological structures during training~\citep{hofer2017, hofer2019, chen2019, shit2023}. \citet{clough2019} used the Betti numbers of the ground truth as a topological prior and computed gradients to adjust the persistence of topological features. \citet{hu2019} introduced a penalty for discrepancies between persistence diagrams of the prediction and ground truth. More recently, \citet{hu2021topology} used discrete Morse theory to identify topologically significant structures and employed persistence-based pruning to refine them.

\citet{turaga2009maximin} introduced MALIS for affinity-based models to improve the predictions at maximin edges. This method involves costly gradient updates requiring a maximin search to identify voxels most likely to cause topological errors. \citet{funke2018large} introduced constrained MALIS to improve computational efficiency by computing gradients in two separate passes: one for affinities within ground truth objects and another for affinities between and outside these objects.

Several studies have leveraged the concept of non-simple voxels to develop topology-aware segmentation methods. \citet{gornet2019reconstructing} improved neuron segmentation by penalizing errors at non-simple voxels. Additionally, homotopy warping has been used to prioritize errors at non-simple voxels over minor misalignments at the boundaries ~\citep{jain2010, hu2022}. However, these approaches are computationally expensive since they require voxel-by-voxel analysis to detect non-simple voxels for each prediction.

\section{Method}
Let $G=(V,E)$ be an undirected graph with the vertex set $V=\{1,\ldots, n\}$. We assume that $G$ is a graphical representation of an image where the vertices represent voxels and edges are defined with respect to a $k$-connectivity constraint\footnote{We assume that $k\in\{4,8\}$ and $k\in\{6,18,26\}$ for 2-d and 3-d images, respectively~\citep{kong1989digital}.}. A ground truth segmentation $y=(y_1,\ldots,y_n)$ is a labeling of the vertices with $y_i\in\{0, \ldots, m\}$. Let $\hat y=(\hat y_1,\ldots,\hat y_n)$ denotes a prediction of the ground truth.

Let $F(y)=\{i\in V: y_i\neq0\}$ be the foreground of the vertex labeling, which may include multiple and potentially touching objects. Let $\mathcal S(y)\subseteq\mathscr P(V)$ be the set of connected components induced by the labeling $y$, where $\mathscr P(V)$ is the power set of $V$. 

In a labeled graph, the connected components are determined by the equivalence relation that $i\sim j$ if and only if $y_i=y_j$ with $i,j\in F(y)$ and there exists a path from $i$ to $j$ that is entirely contained within the same segment. An equivalence relation induces a partition over a set into equivalence classes that correspond to the connected components in this setting.

We propose a novel connectivity-preserving loss function to train topology-aware neural networks with the goal of avoiding false splits of, and false merges between foreground objects.
\begin{defn}\label{def:topoloss_voxels}
Let $\mathcal L:\mathbb R^n\times\mathbb R^n\rightarrow\mathbb R$ be the loss function given by
\begin{align*}
    \mathcal L(y,\hat y)&=(1-\alpha)\,\mathcal L_0(y,\hat y)
    +\alpha\,\beta\sum_{C\in\mathcal P(\hat y_+)}\mathcal L_0(y_C, \hat y_C) \\
    &+\alpha\,(1-\beta)\sum_{C\in\mathcal N(\hat y_\shortminus)}\mathcal L_0(y_C, \hat y_C)
\end{align*}
such that $\alpha,\beta\in[0,1]$ and $\mathcal L_0$ is an arbitrary loss function.
\end{defn}

We build upon a traditional loss function $\mathcal L_0$ (e.g. cross-entropy or Dice coefficient) by incorporating additional terms that penalize sets of connected voxels (i.e. supervoxels) responsible for connectivity errors. These supervoxels are identified by analyzing connected components in the false negative and false positive masks, which are obtained by comparing the foregrounds of $y$ and $\hat y$

The sets $\mathcal N(\hat y_\shortminus)$ and $\mathcal P(\hat y_+)$ consist of connected components whose addition or removal changes the number of connected components. A component that changes the underlying topology in this manner is called a \emph{critical} component. Next, we rigorously define this notion, and present an algorithm for detecting critical components.

\subsection{Critical Components}

Critical components generalize the notion of non-simple voxels from digital topology to supervoxels. In digital topology, a voxel is called non-simple if its addition or removal changes the number of connected components, holes, or cavities. Similarly, a supervoxel is called critical if its addition or removal changes the number of connected components. We use the term \emph{critical} as opposed to \emph{non-simple} because the definition is not a direct generalization for computational reasons. This more focused definition enables our supervoxel-based loss function to be computed in linear time, which is a significant improvement over existing topological loss functions.

\subsubsection{False Splits}
Let $\hat y_\shortminus$ be the false negative mask determined by comparing a prediction to the ground truth. Let $\mathcal S_y(\hat y_\shortminus)$ be the set of connected components of $\hat y_\shortminus$ with respect to $y$. The connected components in this set are defined by the following criteria: two voxels $i\sim j$ belong to the same component if:
\begin{quote}
    \begin{itemize}
        \item[(i)] $(\hat y_\shortminus)_i=(\hat y_\shortminus)_j$ with $i,j\in F(\hat y_\shortminus)$
        \item[(ii)] $y_i=y_j$
        \item[(iii)] There exists a path from $i$ to $j$ that lies  within the same segments.
    \end{itemize}
\end{quote}
This second condition ensures that each component in the false negative mask corresponds precisely to one connected component in the ground truth\footnote{Note that in binary segmentation, $\mathcal S(\hat y_\shortminus)$ is interchangeable with $\mathcal S_y(\hat y_\shortminus)$.}.

\emph{Negatively} critical components are determined by comparing the number of connected components in $y$ and $y\ominus C$ such that $C\in\mathcal S_y(\hat y_\shortminus)$. The notation $y\ominus C$ denotes ``removing'' a component from the ground truth. The result of this operation is a vertex labeling where the label of node $i\in V$ is given by
\begin{equation}\label{eq:remove_component}
    (y\ominus C)_i = \begin{cases}
        0, & \text{if } i\in C \\
        y_i, & \text{otherwise } \\
    \end{cases}
\end{equation}
The removal of a component only impacts a specific region within the graph; the component itself and the nodes connected to its boundary. Thus, for topological characterization, it is sufficient to check whether its removal changes the number of connected components in that local region instead of the entire graph. Let $N(C)\subseteq V$ be the neighborhood surrounding a component $C\in\mathcal S(y)$ such that $N(C)=\{i\in V : \{i, j\}\in E \text{ and } j\in C \}$. Let $y\cap N(C)$ represent the labeling $y$ within $N(C)$.

\begin{figure}[htbp!]
    \centering
    \includegraphics[width=0.95\columnwidth]{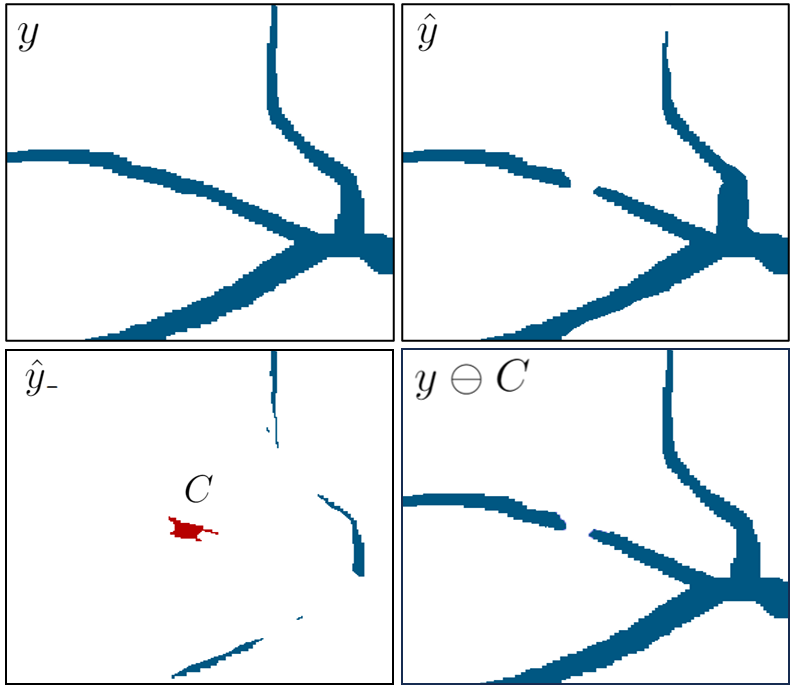}
    \caption{{\bf Top:} Image patches of ground truth and predicted segmentations. {\bf Bottom:} False negative mask with the component $C$ highlighted in red. $C$ is negatively critical since its removal changes the connectivity of the ground truth.} 
    \label{fig:neg-critical}
\end{figure}

\begin{defn}\label{def:neg_critical}
    A component $C\in\mathcal S_y(\hat y_\shortminus)$ is said to be negatively critical if $\vert\mathcal S\big(y \cap N(C)\big)\vert\neq \vert\mathcal S\big((y\ominus C)
    \cap N(C)\big)\vert$.
\end{defn}

Negatively critical components change the local topology by either deleting an entire component or altering the connectivity between vertices in $N(C)$ from the ground truth. In the latter case, the removal of such components locally disconnects some $i,j\in N(C)$ so that it is impossible to find a path (in this neighborhood) that does not pass through $C$. Based on this intuition, we can establish an equivalent definition of negatively critical components as components that are either (1) identical to a component in the ground truth or (2) locally disconnect at least one pair of nodes in $N(C)$ after being removed.

\begin{thm}\label{thm:neg_critical_general}
    A component $C\in\mathcal S_y(\hat y_\shortminus)$ is negatively critical if and only if there exists an $A\in\mathcal S(y\cap N(C))$ with $A\supseteq C$ such that either $(1)$ $A=C$ or $(2)$ $\exists\, v_0, v_k\in A\setminus C$ such that there does not exist a path $(v_0,\ldots, v_k)\subseteq N(C)$ with $v_i\notin C$ for $i=1,\ldots, k-1$. (Proof is in the Appendix.) 
\end{thm}

One computational challenge in both characterizations is the need to recompute connected components within the neighborhood $N(C)$ for every $C\in\mathcal S_y(\hat y_\shortminus)$. In the worst case, the complexity is $\mathcal O(n^2)$ with respect to the number of voxels in the image. However, we can develop a more efficient algorithm with $\mathcal O(n)$ complexity by leveraging two useful facts: (i) neurons are tree-structured objects, implying that, (ii) negatively critical components change both the local \emph{and} global topology. 

Recall that a negatively critical component $C\in\mathcal S_y(\hat y_\shortminus)$ changes the local topology of $N(C)$ in the sense that $\vert\mathcal S\big(y \cap N(C)\big)\vert\neq \vert\mathcal S\big((y\ominus C)\cap N(C)\big)\vert$. Analogously, $C$ also changes the global topology if $\vert\mathcal S(y)\vert\neq \vert\mathcal S(y\ominus C)\vert$. In this special case, we can establish an equivalent definition, similar to Theorem \ref{thm:neg_critical_general}, that utilizes $\mathcal S(y)$ and $\mathcal S(\removalC)$ in place of $\mathcal S(y\cap N(C))$ and $\mathcal S((y\ominus C)\cap N(C))$.

This characterization can be further simplified by using $\mathcal S(\removal)$ instead of $\mathcal S(\removalC)$, where $\removal$ represents the ground truth after removing all components in the false negative mask:
\begin{equation*}
    (\removal)_i = \begin{cases}
        0, & \text{if } (\hat y_\shortminus)_i=1 \\
        y_i, & \text{otherwise } \\
    \end{cases}
\end{equation*}

Note that this characterization is the key to overcoming nonlinear computational complexity.
\begin{cor}\label{cor:neg_critical_special}
    A component $C\in\mathcal S_y(\hat y_\shortminus)$ is negatively critical with $\vert\mathcal S(y)\vert\neq \vert\mathcal S(\removalC)\vert$ if and only if there exists an $A\in\mathcal S(y)$ with $A\supseteq C$ such that either $(1)$ $A=C$ or $(2)$ $\exists B_1,B_2\in\mathcal S(\removal)$ with $B_1,B_2\subset A$ such that $B_1\cup C\cup B_2$ is connected. (Proof is in the Appendix.) 
\end{cor}

\subsubsection{False Merges}

Let $\hat y_+$ be the false positive mask determined by comparing the prediction to the ground truth. Analogously, a component in the false positive mask is positively critical if its \emph{addition} to the ground truth changes the number of connected components in the immediate neighborhood. While this notion can be articulated, for consistency and to leverage previous results, we opt for an equivalent formulation.  Alternatively, a component in the false positive mask is positively critical if its \emph{removal} from the predicted segmentation changes the topology.

\begin{defn}
    A component $C\in\mathcal S_y(\hat y_+)$ is said to be positively critical if $\vert\mathcal S(\hat y\cap N(C))\vert \neq \vert\mathcal S(\additionC\cap N(C))\vert$.
\end{defn}

\begin{figure}[htbp!]
    \centering
    \includegraphics[width=0.95\columnwidth]{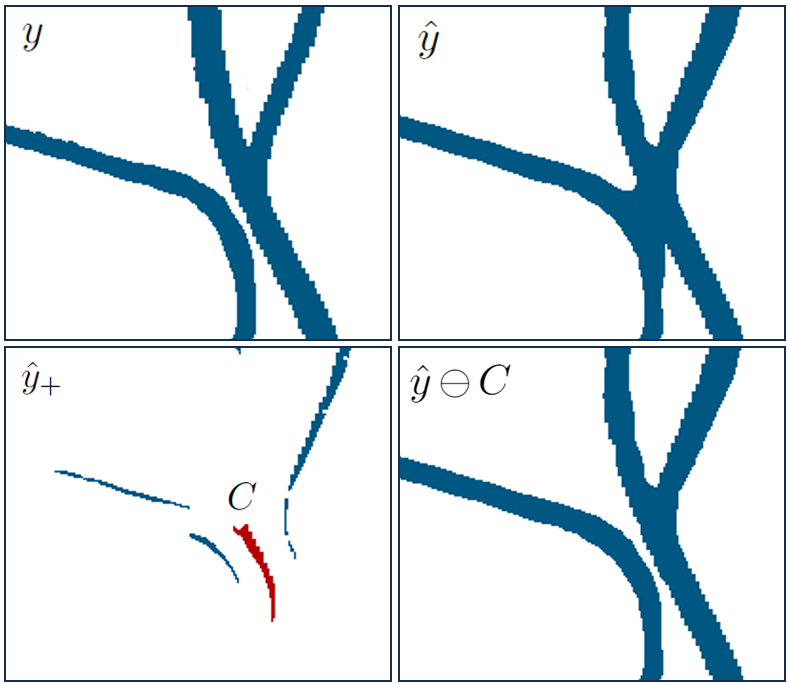}
    \caption{{\bf Top:} Image patches of ground truth and predicted segmentations. {\bf Bottom:} False positive mask with a single component $C$ highlighted. $C$ is positively critical since its removal changes the number of connected components.} 
    \label{fig:pos-critical}
\end{figure}

Positively critical components change the local topology by either (1) creating a component or (2) altering the connectivity between ground truth objects. In the latter, these components connect pairs of nodes that belong to locally distinct components. Equivalently, their removal causes pairs of nodes to become locally disconnected. Next, we present an equivalent definition that characterizes positively critical components as satisfying one of these conditions.

\begin{thm}\label{thm:pos_critical_general}
    A component $C\in\mathcal S_y(\hat y_+)$ is positively critical if and only if there exists an $A\in\mathcal S(\hat y)$ with $A\supseteq C$ such that either $(1)$ $A=C$ or $(2)$ $\exists\, v_0, v_k\in A\setminus C$ such that there does not exist a path $(v_0,\ldots, v_k)\subseteq N(C)$ with $v_i\notin C$ for $i=1,\ldots, k-1$. (Proof is in the Appendix.)
\end{thm}

Similarly, positively critical components present the same computational challenge of needing to recompute connected components for every $C\in\mathcal S_y(\hat y_+)$. However, we can avoid this expensive calculation by utilizing a corollary of Theorem \ref{thm:pos_critical_general} that establishes an equivalent definition of positively critical components that also change the global topology. This characterization uses $\mathcal S(\hat y)$ and $\mathcal S(\addition)$, instead of $\mathcal S(y\cap N(C))$ and $\mathcal S((\additionC)\cap N(C))$, where $\addition$ denotes removing every component in the false positive mask from the prediction via
\begin{equation*}
    (\addition)_i = \begin{cases}
        0, & \text{if } (\hat y_+)_i=1 \\
        \hat y_i, & \text{otherwise } \\
    \end{cases}
\end{equation*}

\begin{cor}\label{cor:pos_critical_special}
    A component $C\in\mathcal S_y(\hat y_+)$ is positively critical with $\vert\mathcal S(\hat y)\vert\neq \vert\mathcal S(\additionC)\vert$ if and only if there exists an $A\in\mathcal S(\hat y)$ with $A\supseteq C$ such that either $(1)$ $A=C$ or $(2)$ $\exists B_1,B_2\in\mathcal S(\addition)$ with $B_1,B_2\subset A$ such that $B_1\cup C\cup B_2$ is connected. (Proof is in the Appendix.)
\end{cor}

\subsection{Computing Critical Components}\label{sec:compute}

Although topological loss functions improve segmentation accuracy, one major drawback is that they are computationally expensive. A key advantage of our proposed method is that the runtime is $\mathcal O(n)$ with respect to the number of voxels. In contrast, related methods have either $\mathcal O(n\log n)$ or $\mathcal O(n^2)$ complexity (e.g.~\citet{turaga2009maximin, gornet2019reconstructing, hu2021topology, shit2021, hu2023learn}). 

In the case of identifying non-simple voxels, Bertrand and Malandain (1994) proved that it is sufficient to examine the topology of the neighborhood. Similarly, we can determine whether a component is critical by checking the topology of nodes connected to the boundary. For the remainder of this section, we focus the discussion on computing negatively critical components since the same algorithm can be used to compute positively critical components.

Let $D(C)=N(C)\setminus C$ be the set of nodes connected to the boundary of a component $C\in \mathcal S_y(\hat y_\shortminus)$. Assuming that a negatively critical component also changes the global topology, Corollary \ref{cor:neg_critical_special} can be used to establish analogous conditions on the set $D(C)$ that are useful for fast computation.

\begin{cor}\label{cor:bfs}
    A component $C\in\mathcal S_y(\hat y_\shortminus)$ is negatively critical with $\vert\mathcal S(y)\vert\neq \vert\mathcal S(\removalC)\vert$ if and only if $\exists A\in\mathcal S(y)$ with $A\supseteq C$ such that either $(1)$ $\nexists\, i\in D(C)$ with $i\in A$ or $(2)$ $\exists B_1, B_2\in\mathcal S(\removal)$ with $B_1,B_2\subset A$ such that $i\in B_1$ and $j\in B_2$ for some $i,j \in D(C)$. (Proof is in the Appendix.)
\end{cor}

The key to achieving linear complexity is to precompute $\mathcal S(y)$ and $\mathcal S(\removal)$, then use a breadth-first search (BFS) to compute $\mathcal S_y(\hat y_\shortminus)$ while simultaneously checking Conditions 1 and 2. Intuitively, the core idea is that once this BFS reaches the boundary of the component, we can visit all nodes in $D(C)$ and efficiently check the conditions.

Let $r\in F(\hat y_\shortminus)$ be the root of the BFS. Given a node $j\in D(C)$, Conditions 1 and 2 can be efficiently checked using a hash table that stores the connected component label of $j$ in $\mathcal S(y)$ and $\mathcal S(\removal)$ as a key-value pair. If we never visit a node $j\in D(C)$ with the same ground truth label as the root, then this label is not a key in the hash table and the component satisfies Condition 1 (Line 19, Algo. 2). 

Now consider the case where we visit a node $j\in D(C)$ with the same ground truth label as the root. If the label of $j$ in $\mathcal S(y)$ is not already a key in the hash table, a new entry is created (Line 15, Algo. 2). Otherwise, the value corresponding to this key is compared to the label of $j$ in $\mathcal S(\removal)$. If they differ, then the component satisfies Condition 2.

Note that pseudocode for this method is provided in Algo.~1 and~2 and our code is publicly available at \url{https://github.com/AllenNeuralDynamics/supervoxel-loss}.

\begin{thm}\label{thm:complexity-01}
    The computational complexity of computing critical components that satisfy either $\vert\mathcal S(y)\vert\neq \vert\mathcal S(\removalC)\vert$ or $\vert\mathcal S(\hat y)\vert\neq \vert\mathcal S(\additionC)\vert$ is $\mathcal O(n)$ with respect to the number of voxels in the image. (Proof is in the Appendix.)
\end{thm}

We emphasize that the statements and algorithms surrounding Theorem \ref{thm:complexity-01} are restricted to tree-structured objects (i.e. critical components that satisfy $\vert\mathcal S(y)\vert\neq \vert\mathcal S(\removalC)\vert$ or $\vert\mathcal S(\hat y)\vert\neq \vert\mathcal S(\additionC)\vert$). Indeed, a similar algorithm based on the main definitions and deductions can be implemented in a straightforward manner for the general case, except that this algorithm will be super-linear in complexity. 
\begin{figure}[htbp!]
    \centering
    \includegraphics[width=0.95\columnwidth]{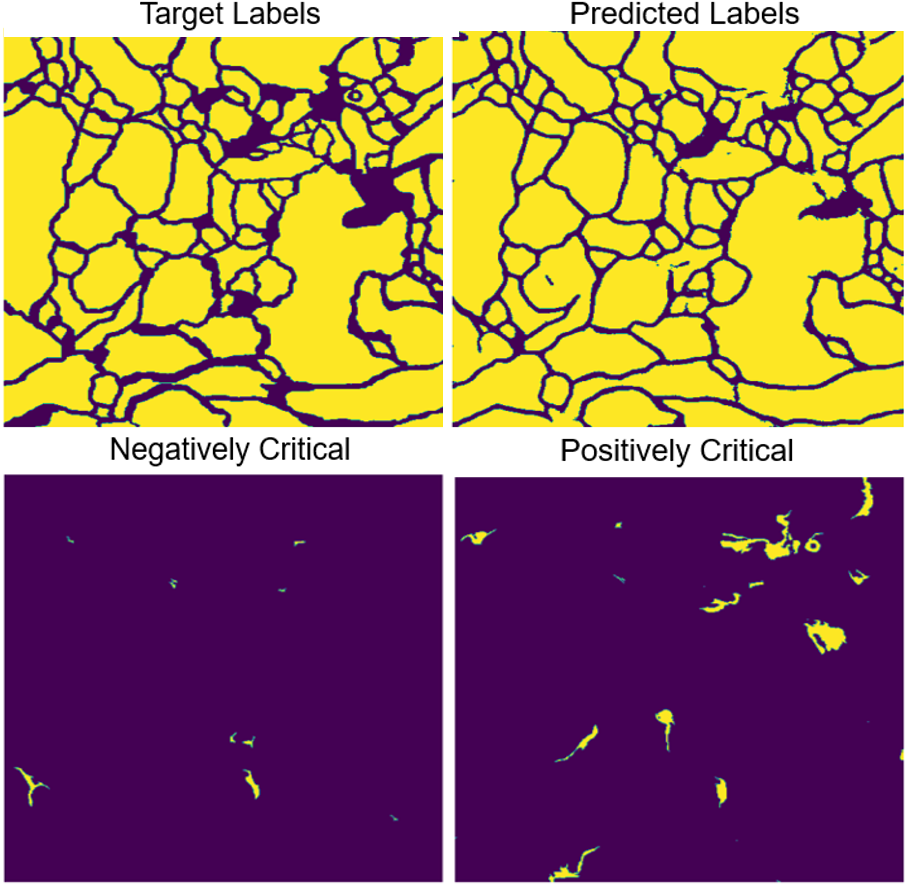}
    \caption{{\bf Segmentation of a 512x512 image from the ISBI12 dataset.} The critical components of this prediction were computed in 1.41 seconds using code at https://github.com/AllenNeuralDynamics/supervoxel-loss.}
    \label{fig:runtime}
\end{figure}

\subsection{Penalizing Critical Topological Mistakes}
Our topological loss function builds upon classical, voxel-based loss functions by adding terms that penalize critical components. The paradigm shift here is to evaluate each mistake at a ``structure level'' that transcends rectilinear geometry, as opposed to the voxel level, without resorting to expensive search iterations. In standard loss functions, mistakes are detected at the voxel-level by directly comparing the prediction at each voxel against the ground truth. Instead, we consider the context of each mistake by determining whether a given supervoxel causes a critical topological mistake.

The hyperparameter $\alpha\in[0,1]$ is a scaling factor that controls the relative importance of voxel-level versus structure-level mistakes. Similarly, $\beta\in[0,1]$ controls the weight placed on split versus merge mistakes. In situations where false merges are more costly or time-consuming to correct manually, $\beta$ can be set to a value greater than 0.5 to prioritize preventing merge errors by assigning them higher penalties.

Our topological loss function is architecture agnostic and can be easily integrated into existing deep learning pipelines. We recommend training a baseline model with a standard loss function, then fine-tuning with the topological loss function. This topological function adds little computational overhead since the only additional calculation is to compute the critical components. In Theorem \ref{thm:complexity-01}, we prove that Algorithms~1 and~2 can be used to compute critical components in linear time. This result can then be used to show that the computational complexity of computing our proposed topological loss function is also $\mathcal O(n)$ in the number of voxels in the image.

\section{Experiments}
\label{sec:results}

We evaluate our method on the following image segmentation datasets: \textbf{DRIVE}, \textbf{ISBI12}, \textbf{CrackTree} and \textbf{EXASPIM}\footnote{Download at s3://aind-msma-morphology-data/EXASPIM25}. DRIVE is a retinal vessel dataset consisting of 20 images with dimensions 584x565~\citep{staal04}. ISBI12 is an electron microscopy (EM) dataset consisting of 29 images with dimensions 512x512. CrackTree contains 206 images of cracks in roads, where the size of each image is 600x800.  EXASPIM is a 3-d light sheet microscopy dataset consisting of 37 images whose dimensions range from 256x256x256 to 1024x1024x1024 and voxel size is $\sim$\SI{1}{\micro\metre}$^3$~\citep{glaser2024}. 

\begin{figure}[H]
    \centering
    \includegraphics[width=0.9\columnwidth]{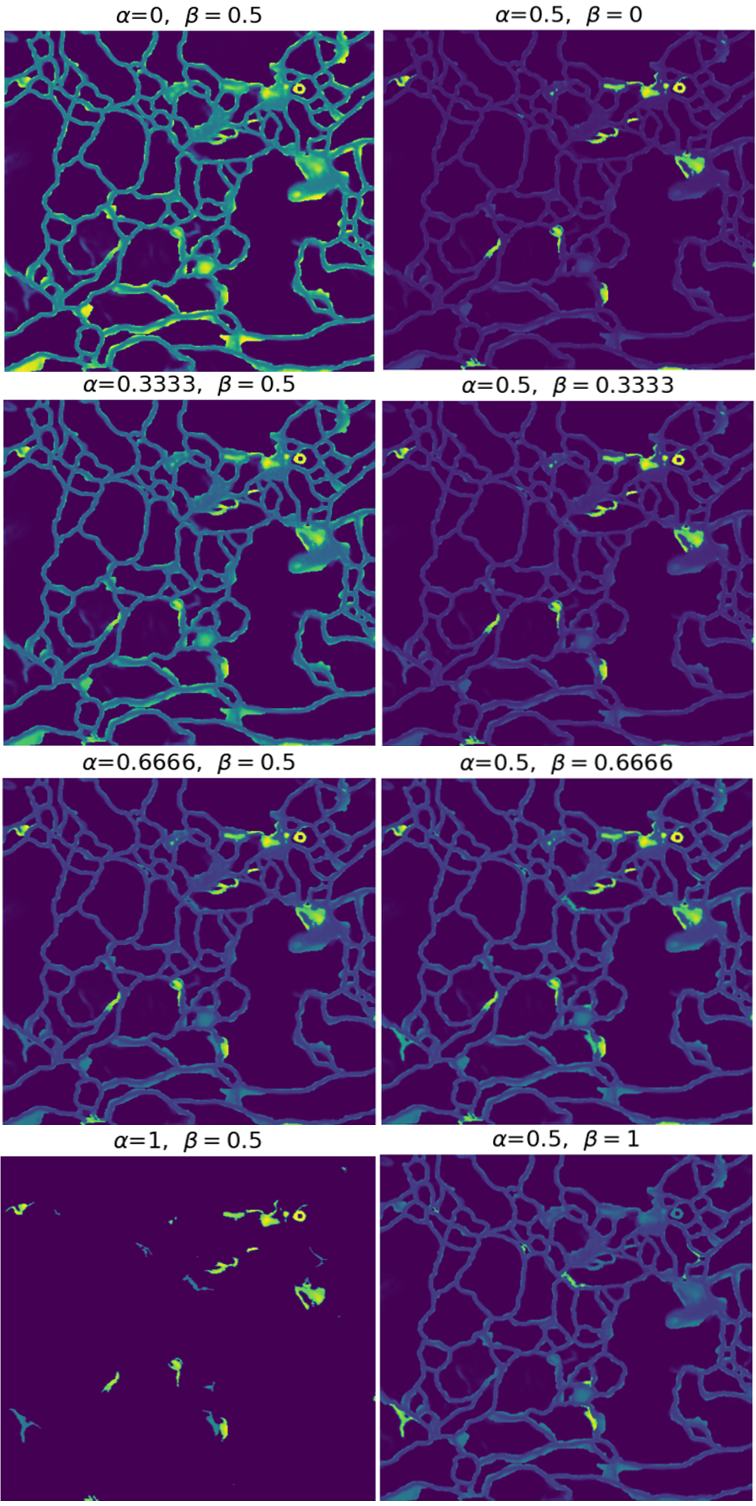}
    \caption{
    {\bf Visualization of the loss in the prediction in Fig.~\ref{fig:runtime}.} \textbf{Left:} As $\alpha$ varies from 0 to 1, the loss places higher penalties on critical components. \textbf{Right:} As $\beta$ varies from 0 to 1, the loss shifts from assigning higher penalties to negatively critical to positively critically components.}
    \label{fig:alpha-beta}
\end{figure}

For the 2-d datasets, we perform 3-fold cross-validation for each method and report the mean and standard deviation across the validation set. For the 3-d dataset, we evaluate the methods on a test set consisting of 4 images. In all experiments, we set $\alpha=0.5$ and $\beta=0.5$ in our proposed topological loss function.

{\bf Evaluation Metrics.} We use four evaluation metrics: pixel-wise accuracy, Dice coefficient, Adjusted Rand Index (ARI), Variation of Information (VOI) and Betti number error. The last three metrics are more topology-relevant in the sense that small topological differences can lead to significant changes in the error. See the Appendix for additional details on each metrics.

{\bf Baselines.} For the 2-d datasets, we compare our method to \textbf{U-Net} \cite{ronneberger2015}, \textbf{Dive} \cite{fakhry2016},  \textbf{Mosin.} \cite{mosinka2018}, \textbf{TopoLoss} \cite{hu2019}, and \textbf{DMT} \cite{hu2023learn}. For the 3-d datasets, we compare our method to \textbf{U-Net}~\cite{ronneberger2015}, \textbf{Gornet}~\cite{gornet2019reconstructing}, \textbf{clDice}~\cite{shit2021}, and \textbf{MALIS}~\cite{turaga2009maximin}. For the 2-d datasets, the segmentations were generated by applying a threshold of 0.5 to the predicted likelihoods. For the 3-d dataset, where some objects are touching, the segmentations were generated by applying a watershed algorithm to the prediction~\citep{zlateski2015}.

{\bf Results.} Table~\ref{table:common-metrics} shows quantitative results for the different models on the segmentation datasets. Our proposed method achieves state-of-the-art results; particularly, for the topologically relevant metrics. Figures~\ref{fig:qualitative_results_2d_main} and \ref{fig:qualitative_results_3d_main} (full raw images and segmentations in the Appendix) present qualitative results that visually compare the performance of U-Net models trained with cross-entropy versus our proposed loss function. Although the only difference between these two models is the addition of topological terms, there is a clear difference in topological accuracy.

\begin{figure}[t]
    \centering
    \includegraphics[width=0.95\columnwidth]{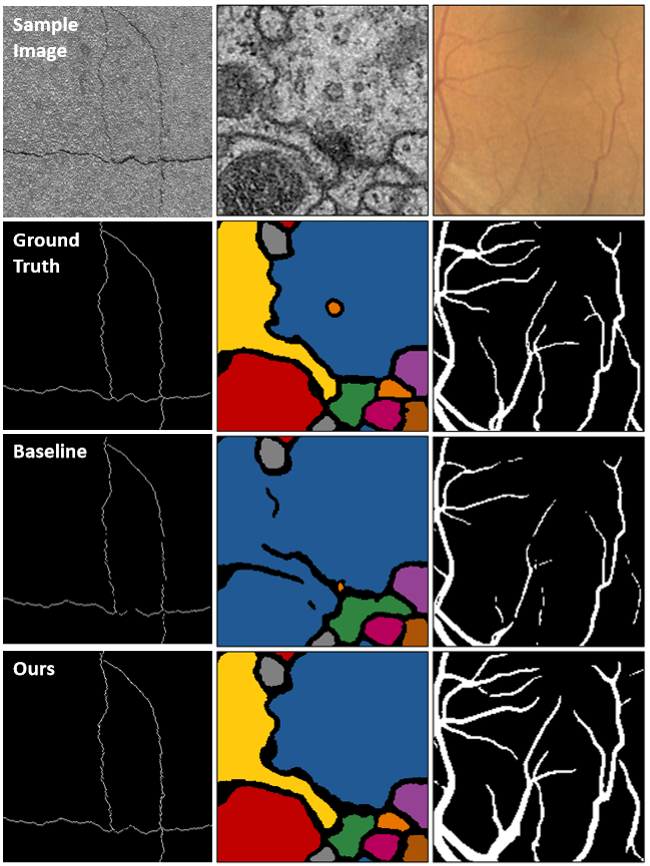}
    \caption{Qualitative results on three 2d datasets.}
    \label{fig:qualitative_results_2d_main}
\end{figure}

\begin{table}[htbp!]
    \centering
    \resizebox{\columnwidth}{!}{%
    \begin{tabular}{llllll}
    \toprule
    Method & $\mathcal O(\cdot)$ & Dice $\uparrow$ & ARI $\uparrow$ & VOI $\downarrow$  & Betti Error $\downarrow$ \\
    \midrule
    \multicolumn{6}{c}{DRIVE} \\
    \midrule
    \multirow{1}{*}{U-Net}
     & $n$ & 0.749$\pm$0.003 & 0.834$\pm$0.041 & 1.98$\pm$0.05 & 3.64$\pm$0.54 \\
    \multirow{1}{*}{DIVE} 
     & $n$ & 0.754$\pm$0.001 & 0.841$\pm$0.026 & 1.94$\pm$0.13 & 3.28$\pm$0.64\\
     \multirow{1}{*}{Mosin.}
     & $n$ & 0.722$\pm$0.001 & 0.887$\pm$0.039 & 1.17$\pm$0.03 & 2.78$\pm$0.29\\
     \multirow{1}{*}{TopoLoss}
     & $n\log n$ & 0.762$\pm$0.004 & 0.902$\pm$0.011 & 1.08$\pm$0.01 & 1.08$\pm$0.27 \\
    \multirow{1}{*}{DMT} 
     & $n^2$ & 0.773$\pm$0.004 & 0.902$\pm$0.002 & 0.88$\pm$0.04 & \textbf{0.87$\pm$0.40} \\
    \multirow{1}{*}{\textbf{Ours}} 
     & $n$ & \textbf{0.809$\pm$0.012} & \textbf{0.943$\pm$0.002} & \textbf{0.48$\pm$0.01} & 0.94$\pm$0.27\\
     \midrule
     \multicolumn{6}{c}{ISBI12} \\
     \midrule
     \multirow{1}{*}{U-Net}
     & $n$ & 0.970$\pm$0.005 & 0.934$\pm$0.007 & 1.37$\pm$0.03 & 2.79$\pm$0.27 \\
     \multirow{1}{*}{DIVE}
     & $n$ & 0.971$\pm$0.003 & 0.943$\pm$0.009 & 1.24$\pm$0.03 & 3.19$\pm$0.31 \\
     \multirow{1}{*}{Mosin.}
     & $n$ & 0.972$\pm$0.002 & 0.931$\pm$0.005 & 0.98$\pm$0.04 & 1.24$\pm$0.25 \\
     \multirow{1}{*}{TopoLoss}
     & $n\log n$ & 0.976$\pm$0.004 & 0.944$\pm$0.008 & 0.78$\pm$0.02 & 0.43$\pm$0.10 \\
     \multirow{1}{*}{DMT}
     & $n^2$ & 0.980$\pm$0.003 & \textbf{0.953$\pm$0.005} & \textbf{0.67$\pm$0.03} & \textbf{0.39$\pm$0.11} \\
     \multirow{1}{*}{\textbf{Ours}}
     & $n$ & \textbf{0.983$\pm$0.001} & 0.934$\pm$0.001 & 0.74$\pm$0.03 & 0.48$\pm$0.02 \\
    \midrule
     \multicolumn{6}{c}{CrackTree} \\
     \midrule
     \multirow{1}{*}{U-Net}
     & $n$ & 0.649$\pm$0.003 & 0.875$\pm$0.042 & 1.63$\pm$0.10 & 1.79$\pm$0.30 \\
     \multirow{1}{*}{DIVE}
     & $n$ & 0.653$\pm$0.002 & 0.863$\pm$0.0376 & 1.57$\pm$0.08 & 1.58$\pm$0.29 \\
     \multirow{1}{*}{Mosin.}
     & $n$ & 0.653$\pm$0.001 & 0.890$\pm$0.020 & 1.11$\pm$0.06 & 1.05$\pm$0.21 \\
     \multirow{1}{*}{TopoLoss}
     & $n\log n$ & 0.673$\pm$0.004 & 0.929$\pm$0.012 & 0.99$\pm$0.01 & 0.67$\pm$0.18 \\
     \multirow{1}{*}{DMT}
     & $n^2$ & \textbf{0.681$\pm$0.005} & \textbf{0.931$\pm$0.017} & \textbf{0.90$\pm$0.08} & 0.52$\pm$0.19 \\
     \multirow{1}{*}{\textbf{Ours}}
     & $n$ & 0.667$\pm$0.010 & 0.914$\pm$0.011 & 0.98$\pm$0.10 & \textbf{0.51$\pm$0.06} \\
    \midrule
    \multicolumn{6}{c}{EXASPIM} \\
    \midrule
     \multirow{1}{*}{U-Net} 
     & $n$ & 0.751$\pm$0.047  & 0.875$\pm$0.082 & 1.28$\pm$0.46 & 0.74$\pm$0.03 \\
     \multirow{1}{*}{Gornet} 
     & $n^2$ & 0.777$\pm$0.083 & 0.901$\pm$0.049 & 0.65$\pm$0.17 & 0.42$\pm$0.07\\
     \multirow{1}{*}{clDice} 
     & $kn$ & 0.785$\pm$0.032 & 0.923$\pm$0.071 & 0.66$\pm$0.51 & 0.36$\pm$0.07 \\
     \multirow{1}{*}{MALIS} 
     & $n^2$ & \textbf{0.794$\pm$0.052} & 0.927$\pm$0.042 & 0.64$\pm$0.27 & 0.34$\pm$0.08\\
     \multirow{1}{*}{\textbf{Ours}} 
     & $n$ & 0.770$\pm$0.058  & \textbf{0.953$\pm$0.038} & \textbf{0.42$\pm$0.21} & \textbf{0.31$\pm$0.06} \\
     \bottomrule
    \end{tabular}
    }
    \caption{{\bf Quantitative results for different models on several datasets.} Results for Dive, Mosin., TopoLoss, and DMT are from~\citep{hu2023learn}. $\mathcal O(\cdot)$: complexity of training iterations, $n$: number of pixels/voxels, $k$: number of pooling operations in clDice.}
    \label{table:common-metrics}
\end{table}

{\bf Evaluation with Skeleton-Based Metrics. }Although the performance of segmentation models is typically evaluated using voxel-based metrics, such as those in Table \ref{table:common-metrics}, skeleton-based metrics are better suited to evaluate neuron segmentations. This is because the main objective of this task is to reconstruct the morphology of individual neurons and their interconnectivity. Thus, skeleton-based metrics are preferable because they quantify to what extent the topological structure of a neuron was accurately reconstructed.

In this evaluation, we compare a set of ground truth skeletons to the predicted segmentations to compute the following metrics: number of splits/neuron (Splits$\slash$Neuron), edge accuracy, and normalized expected run length (ERL) (see the Appendix for definitions). The number of merges per neuron was also computed, but the watershed parameters were configured to prevent merges in the segmentations.

Table \ref{table:skeleton-based-metrics} presents quantitative results for the different models on the EXASPIM dataset. These results show that our proposed method achieves the highest topological accuracy. In particular, our proposed method has the best normalized ERL, which is regarded as the gold standard for evaluating the accuracy of a neuron's topological reconstruction~\citep{jan2018}. Perhaps as importantly, our method scales linearly and achieves favorable runtimes, consistent with Thm. 3 (see Tables \ref{table:skeleton-based-metrics} and \ref{table:full-runtimes}).

\begin{table*}[htbp!]
    \centering
    \resizebox{0.85\textwidth}{!}{%
    \begin{tabular}{lllllll}
    \toprule
    Method & Complexity \hspace{0.5mm} & Runtime$\slash$Epoch $\downarrow$ & Splits$\slash$Neuron $\downarrow$ \hspace{0.5mm}  & Edge Accuracy $\uparrow$ \hspace{0.5mm} & Normalized ERL $\uparrow$ \\
    \midrule
    \multirow{1}{*}{U-Net} & $\mathcal O(n)$ & 10.03$\pm$0.23 sec & 9.86$\pm$13.30\hspace{1mm} & 0.873$\pm$0.087 & 0.596$\pm$0.232 \\
    \multirow{1}{*}{Gornet} & $\mathcal O(n^2)$ & 71.62$\pm$1.83 sec & 3.85$\pm$2.58 & 0.937$\pm$0.062 & 0.664$\pm$0.106\\
    \multirow{1}{*}{clDice} & $\mathcal O(kn)$ & 48.55$\pm$1.60 sec & 3.39$\pm$1.52 & 0.911$\pm$0.042 & 0.701$\pm$0.091\\
    \multirow{1}{*}{MALIS} & $\mathcal O(n^2)$ & 50.68$\pm$1.58 sec & 3.33$\pm$0.59 & 0.917$\pm$0.053 & 0.719$\pm$0.098 \\
    \multirow{1}{*}{\textbf{Ours}} & $\mathcal O(n)$ & 20.12$\pm$1.15 sec & \textbf{2.63$\pm$1.36} & \textbf{0.944$\pm$0.043} & \textbf{0.784$\pm$0.099} \\
    \bottomrule
    \end{tabular}
    }
    \caption{Skeleton-based metrics for different models for the EXASPIM dataset.}
    \label{table:skeleton-based-metrics}
\end{table*}

    \begin{figure*}[htbp!]
        \centering
        \includegraphics[width=0.9\textwidth]{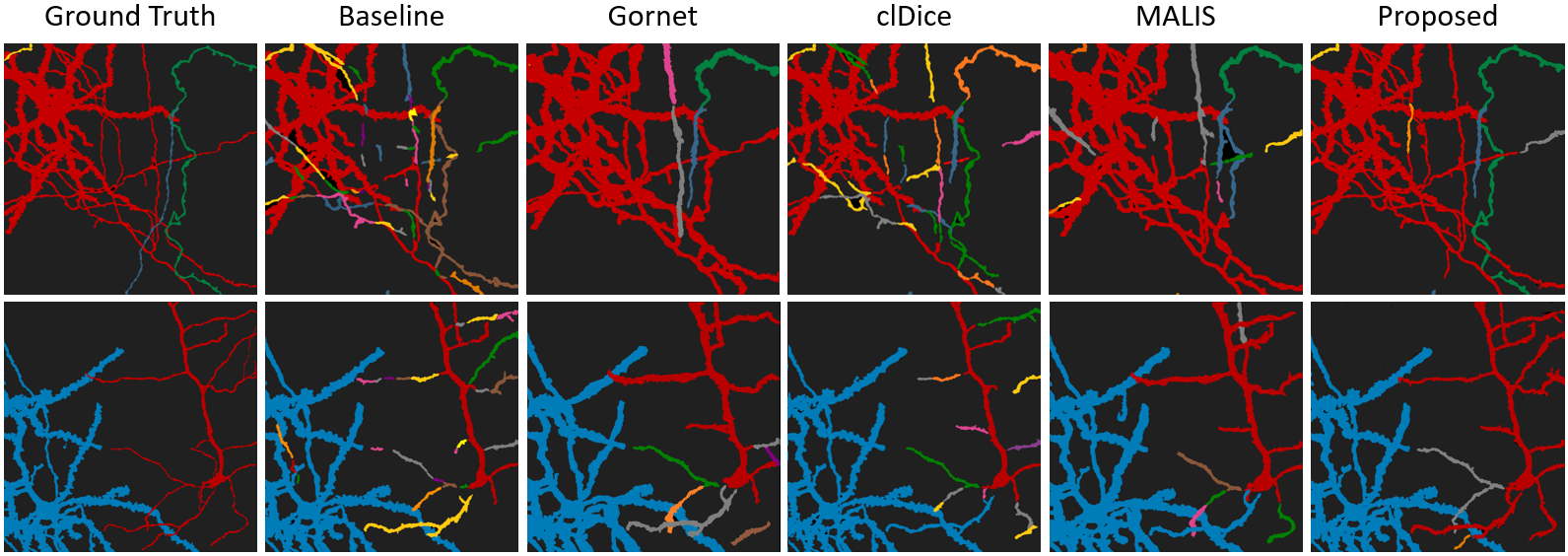}
        \caption{Qualitative results on the 3-d EXASPIM dataset. Raw Image: Fig.~\ref{fig:raw_img}}
        \label{fig:qualitative_results_3d_main}
    \end{figure*}

{\bf Ablation Study of Hyperparameters.} The proposed topological loss function has two interpretable hyperparameters: $\alpha$ and $\beta$. Generally, setting $\alpha$ to a positive value improves the topological accuracy of the segmentation. To better understand their impact on model performance, we performed hyperparameter tuning using Bayesian optimization with the adapted rand index as the objective function~\citep{optuna2019}. We then conducted two separate experiments, where we trained a U-Net using our proposed loss function on both the ISBI12 and EXASPIM datasets.

\begin{figure}[H]
    \centering
    \includegraphics[width=\columnwidth]{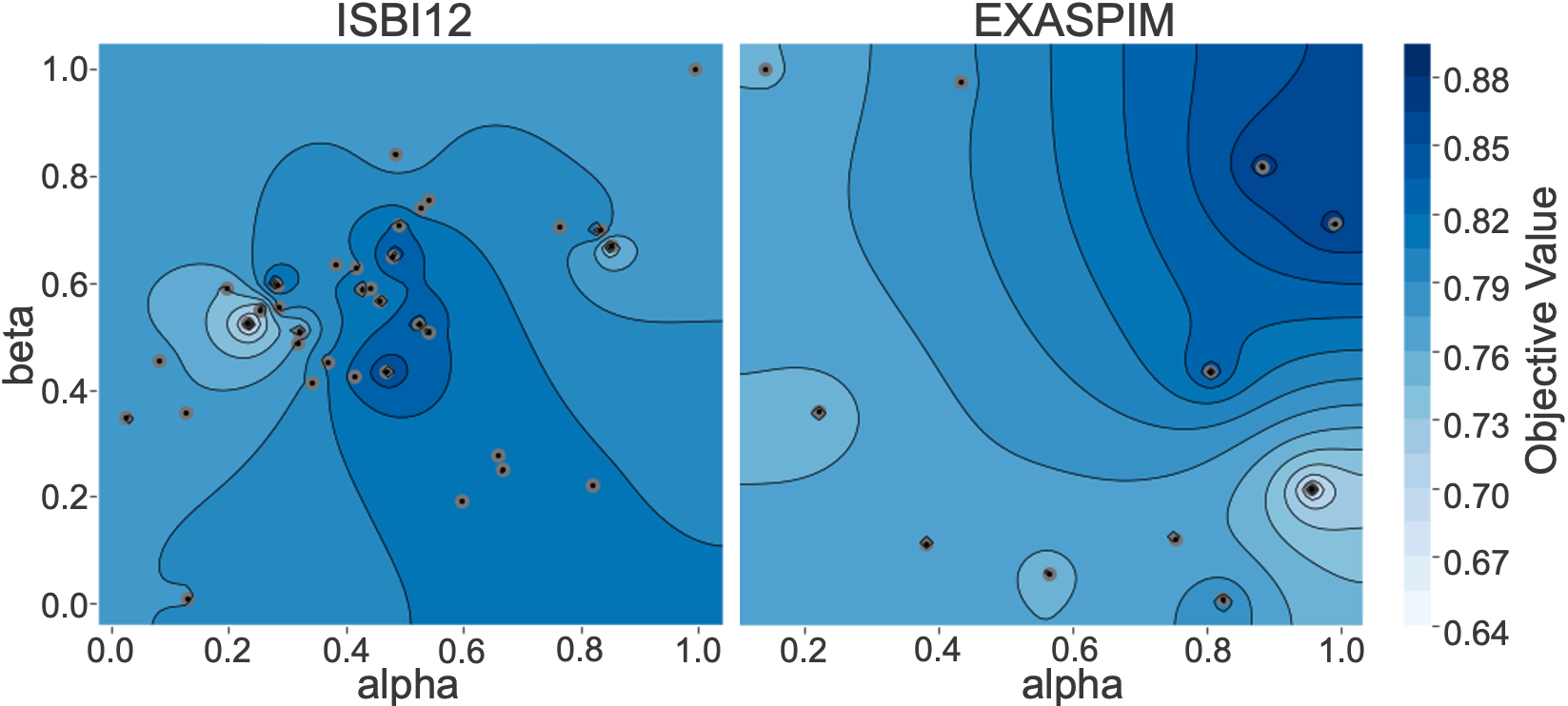}
    \caption{Contour plots of objective functions from hyperparameter optimization, each point is a trial outcome.}
    \label{fig:optuna}
\end{figure}
Although both datasets consist of neuronal images, they are markedly different: the EXASPIM dataset contains sparse images, while ISBI12 consists of dense ones. Our experiments show that $\alpha\approx0.5$ and $\beta\approx0.5$ are optimal for ISBI12, whereas $\alpha\approx0.9$ and $\beta\approx0.8$ are optimal for EXASPIM (Figure~\ref{fig:optuna}). A rule of thumb is to set $\alpha,\beta\approx 0.5$ for dense images and increase them together as the image becomes sparser. Notably, Figure~\ref{fig:optuna} suggests that careful hyperparameter tuning is not critical in practice. In fact, all the results obtained by our method were achieved with a fixed choice for $\alpha$ and $\beta$, without using hyperparameter optimization.

\section{Conclusion}

Mistakes that change the connectivity of objects (i.e. topological mistakes) are a key problem in instance segmentation. They produce qualitatively different results based on relatively few pixels/voxels, making it a challenge to avoid such mistakes with voxel-level objectives. Existing work on topology-aware segmentation typically requires costly steps to guide the segmentation towards correct decisions, among other problems.

Here, we generalize the concept of simple voxel to connected components of arbitrary shape, and propose a topology-aware method with minimal computational overhead. We demonstrate the effectiveness of our approach on multiple datasets with different resolutions, dimensionality, and image characteristics. Across multiple metrics, our method achieves state-of-the-art results. It is now possible to image not only a local patch of neuronal morphology, but whole arbors, which can reach multiple, distant brain regions. The favorable scalability of our approach will enable efficient analysis of such large-scale datasets. 

\bibliography{aaai25}

\begin{thebibliography}{31}
\providecommand{\natexlab}[1]{#1}

\bibitem[{Akiba et~al.(2019)Akiba, Sano, Yanase, Ohta, and Koyama}]{optuna2019}
Akiba, T.; Sano, S.; Yanase, T.; Ohta, T.; and Koyama, M. 2019.
\newblock Optuna: A Next-generation Hyperparameter Optimization Framework.
\newblock In \emph{Proceedings of the 25th {ACM} {SIGKDD} International
  Conference on Knowledge Discovery and Data Mining}.

\bibitem[{Bertrand and Malandain(1994)}]{bertrand1994new}
Bertrand, G.; and Malandain, G. 1994.
\newblock A new characterization of three-dimensional simple points.
\newblock \emph{Pattern Recognition Letters}, 15(2): 169--175.

\bibitem[{Chen et~al.(2019)Chen, Ni, Bai, and Wang}]{chen2019}
Chen, C.; Ni, X.; Bai, Q.; and Wang, Y. 2019.
\newblock A topological regularizer for classifiers via persistent homology.
\newblock In \emph{The 22nd International Conference on Artificial Intelligence
  and Statistics}, 2573--2582. PMLR.

\bibitem[{Clough et~al.(2019)Clough, {\"O}ks{\"u}z, Byrne, Schnabel, and
  King}]{clough2019}
Clough, J.~R.; {\"O}ks{\"u}z, I.; Byrne, N.; Schnabel, J.~A.; and King, A.~P.
  2019.
\newblock Explicit topological priors for deep-learning based image
  segmentation using persistent homology.
\newblock In \emph{Information Processing in Medical Imaging}.

\bibitem[{Fakhry, Peng, and Ji(2016)}]{fakhry2016}
Fakhry, A.; Peng, H.; and Ji, S. 2016.
\newblock Deep models for brain EM image segmentation: novel insights and
  improved performance.
\newblock \emph{Bioinformatics}, 32(15): 2352--2358.

\bibitem[{Funke et~al.(2018)Funke, Tschopp, Grisaitis, Sheridan, Singh,
  Saalfeld, and Turaga}]{funke2018large}
Funke, J.; Tschopp, F.; Grisaitis, W.; Sheridan, A.; Singh, C.; Saalfeld, S.;
  and Turaga, S.~C. 2018.
\newblock Large scale image segmentation with structured loss based deep
  learning for connectome reconstruction.
\newblock \emph{IEEE transactions on pattern analysis and machine
  intelligence}, 41(7): 1669--1680.

\bibitem[{Glaser et~al.(2024)Glaser, Chandrashekar, Vasquez, Arshadi,
  Ouellette, Jiang, Baka, Kovacs, Woodard, Seshamani, Cao, Clack, Recknagel,
  Grim, Balaram, Turschak, Liddell, Rohde, Hellevik, Takasaki, Barner, Logsdon,
  Chronopoulos, de~Vries, Ting, Perlmutter, Kalmbach, Dembrow, Reid, Feng, and
  Svoboda}]{glaser2024}
Glaser, A.; Chandrashekar, J.; Vasquez, J.; Arshadi, C.; Ouellette, N.; Jiang,
  X.; Baka, J.; Kovacs, G.; Woodard, M.; Seshamani, S.; Cao, K.; Clack, N.;
  Recknagel, A.; Grim, A.; Balaram, P.; Turschak, E.; Liddell, A.; Rohde, J.;
  Hellevik, A.; Takasaki, K.; Barner, L.~E.; Logsdon, M.; Chronopoulos, C.;
  de~Vries, S.; Ting, J.; Perlmutter, S.; Kalmbach, B.; Dembrow, N.; Reid,
  R.~C.; Feng, D.; and Svoboda, K. 2024.
\newblock Expansion-assisted selective plane illumination microscopy for
  nanoscale imaging of centimeter-scale tissues.
\newblock \emph{eLife Sciences Publications, Ltd}.

\bibitem[{Gornet et~al.(2019)Gornet, Venkataraju, Narasimhan, Turner, Lee,
  Seung, Osten, and S{\"u}mb{\"u}l}]{gornet2019reconstructing}
Gornet, J.; Venkataraju, K.~U.; Narasimhan, A.; Turner, N.; Lee, K.; Seung,
  H.~S.; Osten, P.; and S{\"u}mb{\"u}l, U. 2019.
\newblock Reconstructing neuronal anatomy from whole-brain images.
\newblock In \emph{2019 IEEE 16th International Symposium on Biomedical Imaging
  (ISBI 2019)}, 218--222. IEEE.

\bibitem[{Gouwens et~al.(2020)Gouwens, Sorensen, Baftizadeh, Budzillo, Lee,
  Jarsky, Alfiler, Baker, Barkan, Berry et~al.}]{gouwens2020integrated}
Gouwens, N.~W.; Sorensen, S.~A.; Baftizadeh, F.; Budzillo, A.; Lee, B.~R.;
  Jarsky, T.; Alfiler, L.; Baker, K.; Barkan, E.; Berry, K.; et~al. 2020.
\newblock Integrated morphoelectric and transcriptomic classification of
  cortical GABAergic cells.
\newblock \emph{Cell}, 183(4): 935--953.

\bibitem[{Hofer et~al.(2019)Hofer, Kwitt, Niethammer, and Dixit}]{hofer2019}
Hofer, C.; Kwitt, R.; Niethammer, M.; and Dixit, M. 2019.
\newblock Connectivity-optimized representation learning via persistent
  homology.
\newblock In \emph{International conference on machine learning}, 2751--2760.
  PMLR.

\bibitem[{Hofer et~al.(2017)Hofer, Kwitt, Niethammer, and Uhl}]{hofer2017}
Hofer, C.; Kwitt, R.; Niethammer, M.; and Uhl, A. 2017.
\newblock Deep learning with topological signatures.
\newblock \emph{Advances in neural information processing systems}, 30.

\bibitem[{Hu(2022)}]{hu2022}
Hu, X. 2022.
\newblock Structure-aware image segmentation with homotopy warping.
\newblock In \emph{Advances in Neural Information Processing Systems
  (NeurIPS)}.

\bibitem[{Hu et~al.(2019)Hu, Li, Samaras, and Chen}]{hu2019}
Hu, X.; Li, F.; Samaras, D.; and Chen, C. 2019.
\newblock Topology-Preserving Deep Image Segmentation.
\newblock In \emph{Advances in Neural Information Processing Systems
  (NeurIPS)}, volume~32, 5657–5668.

\bibitem[{Hu et~al.(2023)Hu, Li, Samaras, and Chen}]{hu2023learn}
Hu, X.; Li, F.; Samaras, D.; and Chen, C. 2023.
\newblock Learning probabilistic topological representations using discrete
  morse theory.
\newblock In \emph{International Conference on Learned Representations (ICLR)}.

\bibitem[{Hu et~al.(2021)Hu, Wang, Li, Samaras, and Chen}]{hu2021topology}
Hu, X.; Wang, Y.; Li, F.; Samaras, D.; and Chen, C. 2021.
\newblock Topology-Aware Segmentation Using Discrete Morse Theory.
\newblock In \emph{International Conference on Learned Representations (ICLR)}.

\bibitem[{Jain et~al.(2010)Jain, Bollmann, Richardson, Berger, Helmstaedter,
  Briggman, Denk, Bowden, Mendenhall, Abraham, Harris, Kasthuri, Hayworth,
  Schalek, Tapia, Lichtman, and Seung}]{jain2010}
Jain, V.; Bollmann, B.; Richardson, M.; Berger, D.~R.; Helmstaedter, M.~N.;
  Briggman, K.~L.; Denk, W.; Bowden, J.~B.; Mendenhall, J.~M.; Abraham, W.~C.;
  Harris, K.~M.; Kasthuri, N.; Hayworth, K.~J.; Schalek, R.; Tapia, J.~C.;
  Lichtman, J.~W.; and Seung, H.~S. 2010.
\newblock Boundary Learning by Optimization with Topological Constraints.
\newblock In \emph{2010 IEEE Computer Society Conference on Computer Vision and
  Pattern Recognition}, 2488--2495.

\bibitem[{Januszewski et~al.(2018)Januszewski, Kornfeld, Li, Pope, Blakely,
  Lindsey, Maitin-Shepard, Tyka, Denk, and Jain}]{jan2018}
Januszewski, M.; Kornfeld, J.; Li, P.; Pope, A.; Blakely, T.; Lindsey, L.;
  Maitin-Shepard, J.; Tyka, M.; Denk, W.; and Jain, V. 2018.
\newblock High-precision automated reconstruction of neurons with flood-filling
  networks.
\newblock \emph{Nature Methods}, 15: 605--610.

\bibitem[{Kong and Rosenfeld(1989)}]{kong1989digital}
Kong, T.; and Rosenfeld, A. 1989.
\newblock Digital topology: Introduction and survey.
\newblock \emph{Computer Vision, Graphics, and Image Processing}, 48(3):
  357--393.

\bibitem[{Lee et~al.(2019)Lee, Turner, Macrina, Wu, Lu, and Seung}]{lee2019}
Lee, K.; Turner, N.; Macrina, T.; Wu, J.; Lu, R.; and Seung, H.~S. 2019.
\newblock Convolutional nets for reconstructing neural circuits from brain
  images acquired by serial section electron microscopy.
\newblock \emph{Current Opinion in Neurobiology}, 55: 188--198.

\bibitem[{Mosinska et~al.(2018)Mosinska, Marquez-Neila, Kozinski, and
  Fua}]{mosinka2018}
Mosinska, A.; Marquez-Neila, P.; Kozinski, M.; and Fua, P. 2018.
\newblock Beyond the Pixel-Wise Loss for Topology-Aware Delineation.
\newblock In \emph{2018 IEEE/CVF Conference on Computer Vision and Pattern
  Recognition (CVPR)}, 3136--3145.

\bibitem[{Ronneberger, Fischer, and Brox(2015)}]{ronneberger2015}
Ronneberger, O.; Fischer, P.; and Brox, T. 2015.
\newblock U-Net: Convolutional Networks for Biomedical Image Segmentation.
\newblock In \emph{Medical Image Computing and Computer-Assisted Intervention
  -- MICCAI 2015}, 234--241. Springer International Publishing.

\bibitem[{Schmidt et~al.(2024)Schmidt, Motta, Sievers, and
  Helmstaedter}]{schmidt2024}
Schmidt, M.; Motta, A.; Sievers, M.; and Helmstaedter, M. 2024.
\newblock RoboEM: automated 3D flight tracing for synaptic-resolution
  connectomics.
\newblock \emph{Nature Methods}, 21: 1--6.

\bibitem[{Sheridan et~al.(2023)Sheridan, Nguyen, Deb, Lee, Saalfeld, Turaga,
  Manor, and Funke}]{sheridan2023}
Sheridan, A.; Nguyen, T.; Deb, D.; Lee, W.-C.~A.; Saalfeld, S.; Turaga, S.;
  Manor, U.; and Funke, J. 2023.
\newblock Local shape descriptors for neuron segmentation.
\newblock \emph{Nautre Methods}, 20: 295--303.

\bibitem[{Shit et~al.(2021)Shit, Paetzold, Sekuboyina, Ezhov, Unger, Zhylka,
  Pluim, Bauer, and Menze}]{shit2021}
Shit, S.; Paetzold, J.; Sekuboyina, A.; Ezhov, I.; Unger, A.; Zhylka, A.;
  Pluim, J.; Bauer, U.; and Menze, B. 2021.
\newblock clDice - a Novel Topology-Preserving Loss Function for Tubular
  Structure Segmentation.
\newblock In \emph{2021 IEEE/CVF Conference on Computer Vision and Pattern
  Recognition (CVPR)}, 16555--16564. IEEE Computer Society.

\bibitem[{Staal et~al.(2004)Staal, Abramoff, Niemeijer, Viergever, and van
  Ginneken}]{staal04}
Staal, J.; Abramoff, M.; Niemeijer, M.; Viergever, M.; and van Ginneken, B.
  2004.
\newblock Ridge-based vessel segmentation in color images of the retina.
\newblock \emph{IEEE Transactions on Medical Imaging}, 23(4): 501--509.

\bibitem[{Stucki et~al.(2023)Stucki, Paetzold, Shit, Menze, and
  Bauer}]{shit2023}
Stucki, N.; Paetzold, J.~C.; Shit, S.; Menze, B.; and Bauer, U. 2023.
\newblock Topologically Faithful Image Segmentation via Induced Matching of
  Persistence Barcodes.
\newblock In \emph{Proceedings of the 40th International Conference on Machine
  Learning}, 32698--32727. PMLR.

\bibitem[{Turaga et~al.(2009)Turaga, Briggman, Helmstaedter, Denk, and
  Seung}]{turaga2009maximin}
Turaga, S.; Briggman, K.; Helmstaedter, M.; Denk, W.; and Seung, S. 2009.
\newblock Maximin affinity learning of image segmentation.
\newblock In \emph{Advances in Neural Information Processing Systems
  (NeurIPS)}, volume~22.

\bibitem[{Turaga et~al.(2010)Turaga, Murray, Jain, Roth, Helmstaedter,
  Briggman, Denk, and Seung}]{turaga2010}
Turaga, S.; Murray, J.; Jain, V.; Roth, F.; Helmstaedter, M.; Briggman, K.;
  Denk, W.; and Seung, S. 2010.
\newblock Convolutional Networks Can Learn to Generate Affinity Graphs for
  Image Segmentation.
\newblock \emph{Neural Computation}, 22(2): 511--538.

\bibitem[{Winnubst et~al.(2019)Winnubst, Bas, Ferreira, Wu, Economo, Edson,
  Arthur, Bruns, Rokicki, Schauder et~al.}]{winnubst2019reconstruction}
Winnubst, J.; Bas, E.; Ferreira, T.~A.; Wu, Z.; Economo, M.~N.; Edson, P.;
  Arthur, B.~J.; Bruns, C.; Rokicki, K.; Schauder, D.; et~al. 2019.
\newblock Reconstruction of 1,000 projection neurons reveals new cell types and
  organization of long-range connectivity in the mouse brain.
\newblock \emph{Cell}, 179(1): 268--281.

\bibitem[{Wu et~al.(2017)Wu, Chen, Wang, Zhang, Yuan, Qian, Metaxas, and
  Axel}]{wu2017}
Wu, P.; Chen, C.; Wang, Y.; Zhang, S.; Yuan, C.; Qian, Z.; Metaxas, D.; and
  Axel, L. 2017.
\newblock Optimal Topological Cycles and Their Application in Cardiac
  Trabeculae Restoration.
\newblock In \emph{Information Processing in Medical Imaging}, 80--92. Springer
  International Publishing.

\bibitem[{Zlateski and Seung(2015)}]{zlateski2015}
Zlateski, A.; and Seung, H.~S. 2015.
\newblock Image Segmentation by Size-Dependent Single Linkage Clustering of a
  Watershed Basin Graph.
\newblock \emph{ArXiv}.

\end{thebibliography}

\appendix

\section{Characterizations of Critical Components}\label{app:char}

Given a graph $G=(V,E)$ and vertex labeling $y$, let $\mathcal S(y\cap N(C))=\{A_1,\ldots, A_m\}$ be the set of connected components induced by this labeling with $A_i\subseteq F(y)$. Let $G^\prime=(V^\prime,E^\prime)$ be the connected component subgraph with $V^\prime=\bigcup_{i=1}^m A_i$ and $E^\prime=\bigcup_{i=1}^mE[A_i]$ where $E[A_i]=\{\{j,k\} \in E: j,k\in A_i\}$. Given a subset $U\subseteq V^\prime$, the subgraph induced by this set is denoted by $G^\prime[U]=(U, E[U])$.

Let $\mu:\mathscr P(G)\rightarrow\mathbb N$ be a function that counts the number of connected components in $G$, where $\mathscr P(G)$ is the set of all subgraphs of $G$. Given $G_1=(V_1, E_1)$ and $G_2=(V_2, E_2)$, a disjoint union of these graphs is defined as $G_1\cup G_2=(V_1\cup V_2, E_1\cup E_2)$.


\begin{lemma}\label{lemma:mu_properties}
    $\mu$ has the following properties:
    \begin{enumerate}[(i), itemsep=0.5ex]
        \item $\mu(G[\varnothing]) = 0$.
        \item Non-negativity. $\mu(G[U])\geq0$ for all $U\subseteq V$.
        \item Finite additivity. For any collection $\{U_i\}^n_{i=1}$ of pairwise disjoint sets with $U_i\subseteq V$,
            \begin{equation*}
\mu\Big(\bigcup_{i=1}^nG[U_i]\Big)=\sum_{i=1}\mu(G[U_i]).
            \end{equation*}
    \end{enumerate}
\end{lemma}
\begin{proof}
    $\mu(G[\varnothing])=0$ because the empty set does not contain any vertices. This function is non-negative by the definition of connected components. For finite additivity, a disjoint union over pairwise disjoint graphs does not affect the connectivity among vertices. Thus, this property holds by using basic set operations in an inductive argument. 
\end{proof}

\subsection{General Case}\label{app:char_general}

\begin{lemma}\label{lemma:subset_identity}
    Given a component $C\in\mathcal S_y(\hat y_\shortminus)$, there exists a unique $A\in\mathcal S(y\cap N(C))$ such that $C\subseteq A$.
\end{lemma}
\begin{proof}
    Choose any $j\in C$, then $y_j\neq0$ by the definition of the false negative mask. Given that $j\in F(y)$, this implies that there exists some $A_k\in\mathcal S(y)$ with $j\in A_k$ and so $C\subseteq \cup_{i=1}^m A_i$. Using this inclusion, $C$ can be decomposed as
    \begin{equation*}
        C = C\cap\bigcup_{i=1}^m A_i= \bigcup_{i=1}^m C\cap A_i.
    \end{equation*}
    This collection of sets is pairwise disjoint since $\{A_i\}_{i=1}^m$ is a collection of connected components.
    
    Next, we claim that there exists a unique $A_k\in\mathcal S(y)$ such that $C\cap A_k\neq\varnothing$. By contradiction, suppose there exists a distinct $A_\ell\in\mathcal S(y)$ with $C\cap A_\ell\neq\varnothing$. But this assumption implies the existence of a path between $A_k$ and $A_\ell$ via $C$ because $C\subseteq F(y\cap N(C))$ implies this set is connected and $C\cap A_k\neq\varnothing$ and $C\cap A_\ell\neq\varnothing$. Since this contradicts $A_k$ and $A_\ell$ being disjoint, $A_k$ must be unique. Lastly, we can use this uniqueness property to conclude that
    \begin{equation*}
        C= \bigcup_{i=1}^m C\cap A_i = C\cap A_k
    \end{equation*}
    which implies that $C\subseteq A_k$.
\end{proof}

\setcounter{thm}{0}
\begin{thm}
    A component $C\in\mathcal S_y(\hat y_\shortminus)$ is negatively critical if and only if there exists an $A\in\mathcal S(y\cap N(C))$ with $A\supseteq C$ such that either $(1)$ $A=C$ or $(2)$ $\exists\, v_0, v_k\in A\setminus C$ such that there does not exist a path $(v_0,\ldots, v_k)\subseteq N(C)$ with $v_i\notin C$ for $i=1,\ldots, k-1$.
\end{thm}
\begin{proof}
    ($\Rightarrow$) Consider the case when $C\in\mathcal S_y(\hat y_\shortminus)$ is negatively critical due to $\vert\mathcal S(y\cap N(C))\vert>\vert\mathcal S((\removalC)\cap N(C))\vert$. Suppose that $\mathcal S(y\cap N(C))=\{A_1, \ldots, A_m\}$, then starting from Equation \ref{eq:remove_component} leads to the identity
    \begin{align}\label{eq:additive}
        \vert\mathcal S((\removalC)\cap N(C))\vert
        &= \mu\Big(\bigcup_{i=1}^m G^\prime[A_i\setminus C] \Big)\nonumber \\
        &= \sum_{i=1}^m\mu(G^\prime[A_i\setminus C]) \nonumber \\
        &= \sum_{\substack{i=1 \\ i\neq j}}^m\mu(G^\prime[A_i]) + \mu(G^\prime[A_j\setminus C])\nonumber\\
        &=\vert\mathcal S(y\cap N(C))\vert-1 \nonumber \\
        &\hspace{4mm}+\mu(G^\prime[A_j\setminus C])
    \end{align}        
    where the second equality holds by $\mu$ being a finitely additive function defined over a collection of pairwise disjoint sets by Lemma \ref{lemma:mu_properties}. The third equality holds by using that there exists a unique $A_j\in\mathcal S(y)$ such that $C\subseteq A_j$ by Lemma \ref{lemma:subset_identity}. Under the assumption that $\vert\mathcal S(y\cap N(C))\vert>\vert\mathcal S((\removalC)\cap N(C))\vert$, it must be the case that $\mu(G^\prime[A_j\setminus C])=0$.
    Thus, we have that $A_j\setminus C=\varnothing$ which implies $A_j\subseteq C$ and so $A_j=C$.

    Next consider the case when $C\in\mathcal S_y(\hat y_\shortminus)$ is negatively critical due to $\vert\mathcal S(y\cap N(C))\vert<\vert\mathcal S((\removalC)\cap N(C))\vert$. Again using the identity in Equation \ref{eq:additive}, the assumed inequality implies that $\mu(G^\prime[A_j\setminus C])\geq2$ and so $G^\prime[A_j\setminus C]$ must contain at least two connected components. Thus, this set can be decomposed into connected components such that
    \begin{equation*}
        A_j\setminus C=\bigcup^K_{k=1}B_k\subseteq \mathcal S((\removalC)\cap N(C))
    \end{equation*}
    with $K\geq2$. For any $v_0\in B_1$ and $v_k\in B_2$, it is impossible to construct a path between these vertices that does not pass through $C$. Otherwise, this would imply that $B_1$ and $B_2$ are path-connected in the graph $G^\prime [A_j\setminus C]$ and not distinct connected components. 

    ($\Leftarrow$) Assume that Condition 1 holds, then $\exists A_j\in\mathcal S(y\cap N(C))$ such that $A_j=C$. The follows immediately by
    \begin{align*}
        \vert\mathcal S((\removalC)\cap N(C))\vert 
        &= \vert\mathcal S(y\cap N(C))\vert-1+\mu(G^\prime[A_j\setminus C]) \\
        &= \vert\mathcal S(y\cap N(C))\vert-1+\mu(G^\prime[C\setminus C]) \\
        &= \vert\mathcal S(y\cap N(C))\vert-1+\mu(G^\prime[\varnothing]) \\
        &= \vert\mathcal S(y\cap N(C))\vert-1
    \end{align*}
    \begin{equation*}
        \implies \vert\mathcal S((\removalC)\cap N(C))\vert < \vert\mathcal S(y\cap N(C))\vert.
    \end{equation*}
    Now assume that Condition 2 holds, then there exists distinct components $B_1, B_2\in\mathcal S((\removalC)\cap N(C))$ with $B_1, B_2\subset A\setminus C$ such that $v_0\in B_1$ and $v_k\in B_2$. Since $B_1, B_2\subset A\setminus C$ are distinct components in the graph $G^\prime[A\setminus C]$, the final result holds by
    \begin{align*}
        \vert\mathcal S((\removalC)\cap N(C))\vert 
        &= \vert\mathcal S(y\cap N(C))\vert-1+\mu(G^\prime[A_j\setminus C]) \\
        &\geq \vert\mathcal S(y\cap N(C))\vert-1 \\
        &\hspace{4mm}+\mu(G^\prime[B_1]\cup G^\prime[B_2]) \\
        &= \vert\mathcal S(y\cap N(C))\vert-1 \\ &\hspace{4mm}+\mu(G^\prime[B_1]) + \mu(G^\prime[B_2]) \\
        &= \vert\mathcal S(y\cap N(C))\vert +1
    \end{align*}
    \begin{equation*}
        \implies \vert\mathcal S((\removalC)\cap N(C))\vert > \vert\mathcal S(y\cap N(C))\vert.
    \end{equation*}
\end{proof}

\subsection{Special Case}\label{app:char_special}

\begin{lemma}\label{lemma:neg_critical_special}
        A component $C\in\mathcal S_y(\hat y_\shortminus)$ is negatively critical if and only if there exists an $A\in\mathcal S(y\cap N(C))$ with $A\supseteq C$ such that either $(1)$ $A=C$ or $(2)$ $\exists\, v_0, v_k\in A\setminus C$ such that there does not exist a path $(v_0,\ldots, v_k)\subseteq N(C)$ with $v_i\notin C$ for $i=1,\ldots, k-1$.
\end{lemma}
\begin{proof}
    The forward direction holds by applying the same argument use to prove Theorem \ref{thm:neg_critical_general}. For the converse, we can again apply the same argument to prove that $\vert \mathcal S(y)\vert\neq\vert\mathcal S(\removalC)\vert$ which then implies that $C$ is negatively critical.
\end{proof}

\begin{lemma}\label{lemma:path-to-union}
    Given a component $C\in\mathcal S_y(\hat y_\shortminus)$ and $A\in\mathcal S(y)$ with $A\supseteq C$, $\exists\, v_0, v_k\in A$ such that there does not exist a path $(v_0,\ldots, v_k)\subseteq A\setminus C$ with $v_i\notin C$ for $i=1,\ldots, k-1$ if and only if $\exists B^\prime_1,B^\prime_2\in\mathcal S(\removalC)$ with $B^\prime_1,B^\prime_2\subset A$ such that $B^\prime_1\cup C\cup B^\prime_2$. 
\end{lemma}

\begin{proof}
    $(\Rightarrow)$ It must be the case that $\hat y_{v_0}=y_{v_0}$ since $\{v_0, v_1\}\in E$ and $v_1\in C$. This implies that $(\removalC)_{v_0}\neq0$ and so there exists some $B^\prime_1\in\mathcal S(\removalC)$ with $v_0\in B^\prime_1$. By the same argument, $(\removalC)_{v_k}\neq0$ and there exists some $B^\prime_2\in\mathcal S(\removalC)$ with $v_k\in B^\prime_2$. Moreover, $v_0$ and $v_k$ must belong to distinct component, i.e. $B^\prime_1\neq B^\prime_2$, since these nodes are not connected in the subgraph induced by $\removalC$. Lastly, $B^\prime_1\cup C\cup B^\prime_2$ is connected due to the existence of the path $(v_0, \ldots, v_k)$ from $B^\prime_1$ to $B^\prime_2$ via $C$.

    $(\Leftarrow)$ The converse holds immediately since $B^\prime_1$ and $B^\prime_2$ are disjoint by definition.
\end{proof}

\begin{lemma}\label{lemma:one-to-all}
    Given a component $C\in\mathcal S_y(\hat y_\shortminus)$ and $A\in\mathcal S(y)$ with $A\supseteq C$, $\exists B^\prime_1,B^\prime_2\in\mathcal S(\removalC)$ with $B^\prime_1,B^\prime_2\subset A$ such that $B^\prime_1\cup C\cup B^\prime_2$ is connected if and only if $\exists B_1,B_2\in\mathcal S(\removal)$ with $B_1,B_2\subset A$ such that $B_1\cup C\cup B_2$ is connected.
\end{lemma}
\begin{proof}
    $(\Rightarrow)$ 
     Given that $B^\prime_1\cup C\cup B^\prime_2$ is connected, there exists a path $(v_0,\ldots,v_k)$ from $B^\prime_1$ to $B^\prime_2$ via $C$ since $B^\prime_1$ and $B^\prime_2$ are disjoint. It must be the case that $\hat y_{v_0} = y_{v_0}$ since $\{v_0, v_1\} \in E$ and $v_1\in C$. This implies that $(y\ominus C)_{v_0}\neq0$ and so $\exists B_1\in \mathcal S(\removal)$ with $v_0\in B_1$. The same argument can be applied to $v_k$ to prove $\exists B_2\in \mathcal S(\removal)$ with $v_k\in B_2$. Thus, the same path that connects the sets $B^\prime_1, B^\prime_2,$ and $A$ also connects  $B_1, B_2,$ and $A$.

    $(\Leftarrow)$ The converse holds by applying the same argument.
\end{proof}

\setcounter{cor}{0}
\begin{cor}
    A component $C\in\mathcal S_y(\hat y_\shortminus)$ is negatively critical with $\vert\mathcal S(y)\vert\neq \vert\mathcal S(\removalC)\vert$ if and only if there exists an $A\in\mathcal S(y)$ with $A\supseteq C$ such that either $(1)$ $A=C$ or $(2)$ $\exists B_1,B_2\in\mathcal S(\removal)$ with $B_1,B_2\subset A$ such that $B_1\cup C\cup B_2$ is connected.
\end{cor}
\begin{proof}
    This result follows immediately by applying Lemmas \ref{lemma:neg_critical_special} - \ref{lemma:one-to-all}.
\end{proof}

\begin{cor}
    A component $C\in\mathcal S_y(\hat y_+)$ is positively critical with $\vert\mathcal S(\hat y)\vert\neq \vert\mathcal S(\additionC)\vert$ if and only if there exists an $A\in\mathcal S(\hat y)$ with $A\supseteq C$ such that either $(1)$ $A=C$ or $(2)$ $\exists B_1,B_2\in\mathcal S(\addition)$ with $B_1,B_2\subset A$ such that $B_1\cup C\cup B_2$ is connected.
\end{cor}
\begin{proof}
    Let $z=\hat y$ and $\hat z_\shortminus=\hat y_+$, then the result follows immediately by applying Corollary \ref{cor:neg_critical_special}.
\end{proof}

\begin{figure}[htbp!]
    \centering
    \includegraphics[width=0.9\columnwidth]{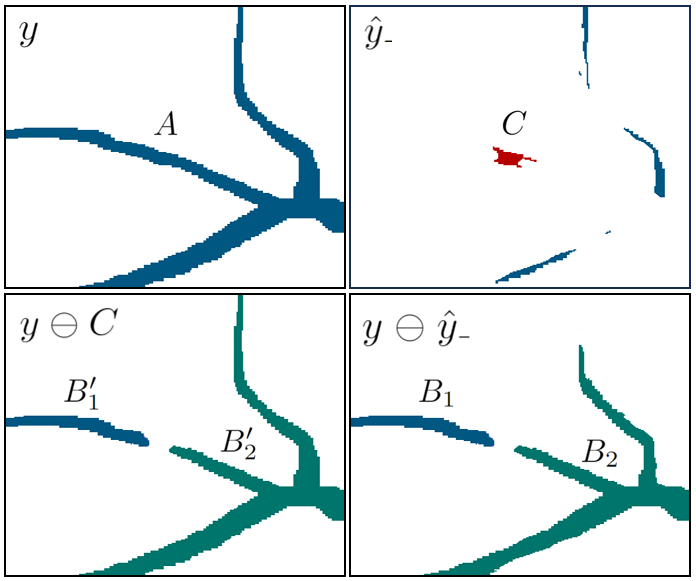}
    \caption{Visualization of Corollary \ref{cor:neg_critical_special}.
    }
\end{figure}

\subsection{Computation}\label{app:compute}

\setcounter{cor}{2}
\begin{cor}
    A component $C\in\mathcal S_y(\hat y_\shortminus)$ is negatively critical with $\vert\mathcal S(y)\vert\neq \vert\mathcal S(\removalC)\vert$ if and only if $\exists A\in\mathcal S(y)$ with $A\supseteq C$ such that either $(1)$ $\nexists\, i\in D(C)$ with $i\in A$ or $(2)$ $\exists B_1, B_2\in\mathcal S(\removal)$ with $B_1,B_2\subset A$ such that $i\in B_1$ and $j\in B_2$ for some $i,j \in D(C)$.
\end{cor}
\begin{proof}
    First, we prove that Condition 1 is equivalent to Condition 1 in Corollary \ref{cor:neg_critical_special}. For the forward direction, Lemma \ref{lemma:subset_identity} implies that $C\subseteq A$. Now choose any $i\in A$, then it must be the case that $i\in C$ since $A$ is connected and $\nexists j\in D(C)$ with $y_i=y_j$. The converse is trivial since $A=C$ implies that
    \begin{equation*}
        D(C)=N(C)\setminus C= N(A)\setminus A
    \end{equation*}
    and so $\nexists i \in D(C)$ with $i\in A$ since the set $A$ is entirely removed from $D(c)$.

    Next, we prove that Condition 2 is equivalent to Condition 2 in Corollary \ref{cor:neg_critical_special}. 
    Let $v_0=i$ and $v_k=j$, there exists a path connecting these vertices contained in $C$ since these nodes are connected to the boundary of $C$ which is a connected set. For the converse, $B_1\cup C\cup B_2$ being connected but $B_1\cup B_2$ being disconnected implies that there exists a path $(v_0,\ldots, v_k)$ from $B_1$ to $B_2$ via $C$. Since this path passes through $C$, it must be the case that $\exists v_i,v_j\in D(C)$ such that $v_i\in B_1$ and $v_j\in B_2$.     
\end{proof}

\begin{thm}
    The computational complexity of computing critical components is $\mathcal O(n)$ with respect to the number of voxels in the image.
\end{thm}
\begin{proof}
    This algorithm involves first precomputing the false negative and false positive masks which can be computed in linear time by comparing each entry in the prediction and ground truth. Next, we must also precompute the following sets of connected components: $\mathcal S(y)$, $\mathcal S(\hat y)$, $\mathcal S(\removal)$, and $\mathcal S(\addition)$.
    Since connected components can be computed in linear time, these precomputations can also be done in linear time~(Cormen et al. 2009). 
    
    Next, a BFS is performed over both $\hat y_\shortminus$ and $\hat y_+$ to extract the connected components of the false negative and/or positive masks. During this BFS, we can determine whether a component satisfies Corollary \ref{cor:bfs} in lines 14-21 in Algorithm 2. A BFS is a linear time algorithm~(Cormen et al. 2009). Since the operations in lines 13-18 can be achieved in constant time, the complexity of this BFS is still linear.
\end{proof}

\begin{algorithm*}
\caption{Detection of Critical Components}
\begin{algorithmic}[1]
\label{alg:critical_components}
    \Procedure{detect\blank{2mm}criticals}{$y,\; \hat y$}:
        \State $\hat y_\shortminus\leftarrow$ compute false negatives
        \State $\mathcal S(y)\leftarrow$ compute connected components
        \State $\mathcal S(\removal)\leftarrow$ compute connected components
        \State $\mathcal N(\hat y_\shortminus) = $ get\blank{2mm}critical($y,\; \hat y_\shortminus,\; \mathcal S(y),\; \mathcal S(\removal)$)
        \State
        \State $\hat y_+ \leftarrow$ compute false positives
        \State $\mathcal S(\hat y) \leftarrow$ compute connected components
        \State $\mathcal S(\addition) \leftarrow$ compute connected components
        \State $\mathcal P(\hat y_+) = $ get\blank{2mm}critical($\hat y,\; \hat y_+,\; \mathcal S(\hat y),\; \mathcal S(\addition)$)
        \State \textbf{return} $\mathcal N(\hat y_\shortminus),\; \mathcal P(\hat y_+)$
    \EndProcedure
    \State \State {\color{officegreen}\# Note that $\hat y_{\times}$ is a placeholder for $\hat y_\shortminus$ and $\hat y_+$}
    \Procedure{get\blank{2mm}critical}{$y,\; \hat y_\times,\; \mathcal S(y),\; \mathcal S(y\ominus \hat y_\times)$}
    \State $F(\hat y_\times) \leftarrow$ compute foreground
    \State $\mathcal X(\hat y_\times)= \text{set}()$
    \State\textbf{while} $\vert F(\hat y_\times)\vert > 0:$
        \State\hspace{5mm} $r =$ sample($F(\hat y_\times)$)
        \State\hspace{5mm} $C,\; is\blank{2mm}critical=$  get\blank{2mm}component($y,\; \hat y_\times,\; \mathcal S(y),\; \mathcal S(y\ominus \hat y_\times),\; r$)
        \State\hspace{5mm} $F(\hat y_\times$).remove($C$)
        \State\hspace{5mm} \textbf{if} $is\blank{2mm}critical:$
        \State\hspace{10mm} $\mathcal X(\hat y_\times)$.add($C$)
    \State \textbf{return} $\mathcal X(\hat y_\times)$
    \EndProcedure
\end{algorithmic}
\end{algorithm*}

\begin{algorithm*}
\caption{Check if Component is Critical}
\label{alg:check_component}
\begin{algorithmic}[1]
    \Procedure{get\blank{2mm}component}{$y,\; \hat y_\times,\; \mathcal S(y),\; \mathcal S(y\ominus \hat y_\times),\; r$\,}:
    \State $C = \text{set}()$
    \State $collisions = \text{dict}()$
    \State $is\blank{2mm}critical = $ False 
    \State $queue = [r]$
    \State\textbf{while} $\vert queue\vert > 0:$

    \State\hspace{5mm} $i= queue.\text{pop}()$
    \State\hspace{5mm} $C.\text{add}(i)$

    \State\hspace{5mm} \textbf{for} $j$ in $N(i)$:
    
    \State\hspace{10mm} \textbf{if } $y_j==y_r$:
    \State\hspace{15mm} \textbf{if } $(\hat y_\times)_j==1$:
    \State\hspace{20mm} $queue.\text{push}(j)$

    \State\hspace{15mm} \textbf{else}:
    \State\hspace{20mm} $\ell_j=$ get\blank{2mm}label($\mathcal S(y),\, j$)
    
    \State\hspace{20mm} \textbf{if } $\ell_j$ not in $collisions.\text{keys}()$:
    \State\hspace{25mm} $collisions[\ell_j]=$ get\blank{2mm}label($\mathcal S(y\ominus \hat y_\times),\, j$)

    \State\hspace{20mm} \textbf{elif } $collisions[\ell_j]\mathrel{\mathtt{!=}}$ get\blank{2mm}label($\mathcal S(y\ominus \hat y_\times),\, j$):
    \State\hspace{25mm} $is\blank{2mm}critical$ = True
    \State \textbf{if } $y_r$ not in $collisions.\text{keys}()$ :
    \State\hspace{5mm} $is\blank{2mm}critical=$  True
    \State \textbf{return} $C,\; is\blank{2mm}critical$
    \EndProcedure
\end{algorithmic}
\end{algorithm*}

\section{Extension to Affinity Models}\label{app:affinity}

An affinity model is a graph-based segmentation method where the main objective is to predict whether neighboring nodes belong to the same segment. Given a graph $G=(V,E)$ and ground truth segmentation $y$, let $\delta:E\rightarrow \{0,1\}$ be the affinity function given by
\begin{equation*}
    \delta(\{i,j\})=\begin{cases}
        1, & \text{if } i,j\in A \\
        0, & \text{otherwise}
    \end{cases}
\end{equation*}
for some $A\in\mathcal S(y)$. One key advantage of affinity models is that they allow instance segmentation to be equivalently formulated as a binary classification task. This formulation is particularly useful when the number of segments is unknown and when distinct objects may be in close contact or even touch.

Neural networks have been used to learn affinities for image segmentation tasks~\citep{turaga2010}. In this approach, the model learns affinity channels, where each channel represents the connectivity between voxels along a certain direction (e.g. vertical or horizontal). The loss function is then computed as a sum over individual losses, each corresponding to one of these affinity channels.

\begin{defn}\label{def:topoloss_edges}
Let $\mathcal L_k:\mathbb R^{nk}\times\mathbb R^{nk}\rightarrow\mathbb R$ be the loss function for an affinity model with $k$ channels be given by
\begin{equation*}
     \mathcal L_k(y,\hat y)=\sum_{i=1}^k\mathcal L(y^{(i)},\hat y^{(i)})
\end{equation*}
where $y^{(i)}, \hat y^{(i)}$ are binary affinities corresponding to the $i$-th channel and $\mathcal L$ is the loss function from Definition~\ref{def:topoloss_voxels}.
\end{defn}

Affinity models involve a transformation between voxel- and edge-based representations of an image. Importantly, critical components are defined within the context of a voxel-based representation. Therefore, each affinity prediction must first be transformed into a voxel-based representation, where the critical components are computed. Afterward, these critical components must be converted back into the affinity-based representation to ensure that topological information is properly incorporated into the segmentation process.

\section{Experiments}

\subsection{Training Protocol}\label{app:training_protocol}

In experiments involving the EXASPIM dataset, we trained all models using binary cross-entropy to learn three affinity channels representing voxel connectivity along the vertical, horizontal, and depth axes. Each model was trained for a total of 2000 epochs with the learning rate $10^{-3}$ and batch size 8. For the topological loss functions, we used the first 300 epochs to train a U-Net model, then fine-tuned the model for the last 1700 epochs.

The affinity predictions were then processed using a watershed algorithm to agglomerate a 3-d over-segmentation derived from the predicted affinities. This segmentation training and inference pipeline was implemented using the following Github repositories: \url{https://github.com/jgornet/NeuroTorch} and \url{https://github.com/funkey/waterz}.

\subsection{Voxel-Based Metrics}\label{app:pixel-based-metrics}

\hspace{8mm}\textbf{Accuracy}: Fraction of correctly labeled voxels.

\hspace{5mm}\textbf{Dice}: Metric that combines precision and recall into a single value to provide a balanced measure of a model's performance.

\hspace{5mm}\textbf{Adjusted Rand Index (ARI)}:
Measures the similarity between two clusterings by comparing the agreement and disagreement in pairwise assignments. In this version of the Rand index, a value of zero is excluded, ensuring a more refined comparison of clustering performance~(Rand 1971).

\hspace{5mm}\textbf{Variation of Information (VOI)}:
 Distance measure between two clusterings that quantifies the amount of information lost or gained when transitioning between clusterings~(Meila 2003).

\subsection{Skeleton-Based Metrics}

Let $G_i=(V_i, E_i)$ be an undirected graph representing the skeleton of the $i$-th object in a ground truth segmentation $y$. Let $\{S_1,\ldots, S_n\}$ denote a collection of such skeletons such that there is a bijection between skeletons and objects in the ground truth segmentation. Each node $u\in V_i$ has a 3-d coordinate $\varphi(u)$ that represents some voxel in $y$. Let $\hat y$ be the predicted segmentation to be evaluated. Finally, let $\hat y[\varphi(u)]$ be the label of voxel $u$ in the predicted segmentation.

The core idea behind skeleton-based metrics is to evaluate the accuracy of a predicted segmentation by comparing it to a set of ground truth skeletons. A key challenge in this process is to ensure that these metrics are robust to minor misalignments between the segmentation and skeletons (e.g., see Node 4 in Figure~\ref{fig:splits-examples}). To address this, a preprocessing step performs a depth-first search to identify pairs of vertices $u,v\in V_i$ that satisfy the following conditions:
\begin{quote}
    \begin{itemize}
        \item[(i)] $\hat y[\varphi(u)]=\hat y[\varphi(v)]$ and $\hat y[\varphi(u)]\neq0$
        \item[(ii)] There exists a path $\{u, w_1,\ldots, w_t, v\}\subseteq V_i$ such that $\hat y[\varphi(w_j)]=0$ for $j=1,\ldots, t$.
    \end{itemize}
\end{quote}
For any pair of vertices $u,v\in V_i$ that satisfy these conditions, we update the predicted segmentation along the detected path by setting $\hat y[\varphi(w_j)]:=\hat y[\varphi(u)]$ for $j=1,\ldots, t$. 

Let $G_i[U_i]$ be the subgraph induced by the vertex subset given by $U_i=\{u\in V_i : \hat y[\varphi(u)]\neq0 \}$. Using this definition, we define the following metric
\begin{equation*}
    \textbf{Splits / Neuron} = \sum\limits_{i=1}^n w_i\big(\big\vert \mathcal S(G[U_i])\big\vert -1\big),
\end{equation*}
where $w_j = \vert E_j\vert\; /\; \sum\limits_{i=1}^n\vert E_i\vert$. Intuitively, this metric computes a weighted average of the number of splits per skeleton.

An edge $\{u, v\}\in E_i$ is said to be \emph{omit} if either $\hat y[\varphi(u)]=0$ or $\hat y[\varphi(v)]=0$. Using this definition, the metric \textbf{\% Omit} is given by
\begin{equation*}
    \textbf{\% Omit} =100\cdot\frac{\sum\limits_{i=1}^n \big\vert\big\{\{u,v\}\in E_i : \{u,v\} \text{ is omit}\big\}\big\vert}{\sum\limits_{i=1}^n \vert E_i\vert}.
\end{equation*}
\begin{figure}[H]
    \centering
    \includegraphics[width=0.85\columnwidth]{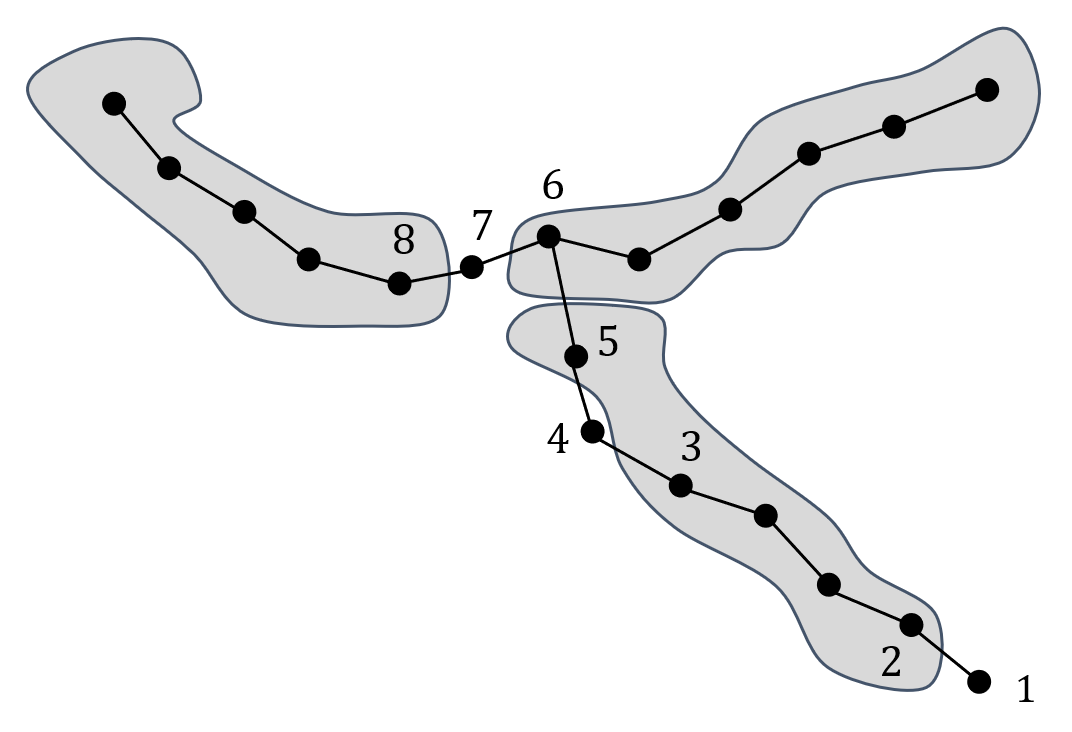}
    \caption{One ground truth skeleton and a predicted segmentation consisting of three connected components. Node 4 is slightly misaligned, its label would be updated to match Nodes 3 and 5. Edge $\{1,2\}$ is omit, but not counted in the metric \textbf{Splits / Neuron}. Edge $\{5,6\}$ is a split and there is a split between nodes 6 and 8. Edges $\{6,7\}$ and $\{7,8\}$ are also omit.}
    \label{fig:splits-examples}
\end{figure}

Given the collection of subgraphs $\{G_1[U_1], \ldots, G_n[U_n]\}$, two connected components $C_1\in\mathcal S(G_i[U_i])$ and $C_2\in\mathcal S( G_j[U_j])$ with $i\neq j$ are said to be \emph{merged} if there exists a pair of nodes $u\in C_1$ and $v\in C_2$ such that $\hat y[\varphi(u)]= \hat y[\varphi(v)]$. Let $M_i[U_i]$ be the set of all such components within the subgraph $G[U_i]$. Using this definition, we define the following metric
\begin{equation*}
    \textbf{\% Merged} = 100\cdot\frac{\sum\limits_{i=1}^n\sum\limits_{ C\in M_i[U_i]} \vert C\vert}{\sum\limits_{i=1}^n \vert E_i\vert}
\end{equation*}

\begin{figure}[H]
    \centering
    \includegraphics[width=0.6\columnwidth]{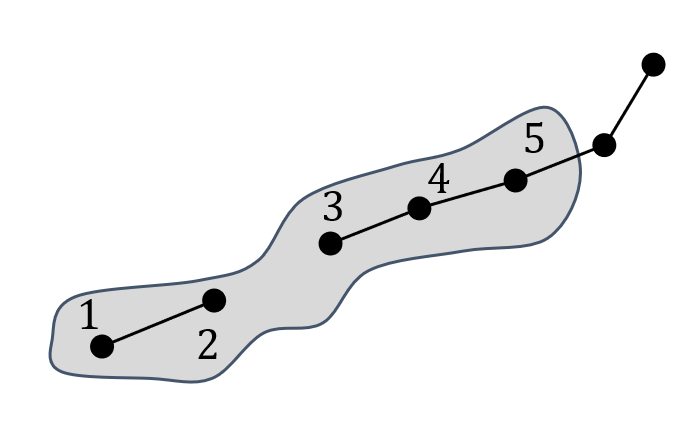}
    \caption{Two ground truth skeleton and a predicted segmentation consisting of one connected component. Edges $\{1,2\}$, $\{3,4\}$, and $\{4,5\}$ are merged.}
    \label{fig:merge-example}
\end{figure}

Using the previous two metrics, \textbf{Edge Accuracy} is defined as
\begin{equation*}
    \textbf{Edge Accuracy} = 100 - (\textbf{\% Omit} + \textbf{\% Merged})
\end{equation*}
Intuitively, edge accuracy represents the percentage of edges from ground truth skeleton that were correctly segmented.

Next, we define expected run length (ERL) metrics which quantify the expected length of a segment corresponding to some uniformly sampled skeleton node. Let $A_i = E_i\setminus(U_i\cup M[U_i])$ be the set of correctly segmented edges of the skeleton $G_i$. Using this definition, we define the following metrics:
\begin{align*}
    \textbf{ERL}_{G_i}&=\sum_{C\in\mathcal S(G_i[A_i])}\vert C\vert \cdot \frac{\vert C\vert}{\vert E_i\vert} \\
    \textbf{Normalized ERL}_{G_i} &= \frac{\textbf{ERL}_{G_i}}{\vert E_i\vert} \\
    \textbf{Normalized ERL}&= \sum_{i=1}^nw_i\cdot \textbf{Normalized ERL}_{G_i},
\end{align*}
where $w_j = \vert E_j\vert\; /\; \sum\limits_{i=1}^n\vert E_i\vert$.

Note that code is publicly available at \url{https://github.com/
AllenNeuralDynamics/segmentation-skeleton-metrics}.

\subsubsection{Additional Experimental Results}\label{app:additional_results}

Figures~\ref{fig:qualitative_results_2d_suppl} and \ref{fig:qualitative_results_3d_suppl} show the predicted segmentations for the full images, from which the image patches shown in Figures~\ref{fig:qualitative_results_2d_main} and \ref{fig:qualitative_results_3d_main} were sampled.

\begin{table*}[htbp!]
    \centering
    \resizebox{0.9\textwidth}{!}{%
    \begin{tabular}{lllllll}
    \toprule
    Method & Complexity & Accuracy $\uparrow$ & Dice $\uparrow$ & ARI $\uparrow$ & VOI $\downarrow$  & Betti Error $\downarrow$ \\
    \midrule
    \multicolumn{6}{c}{DRIVE} \\
    \midrule
    \multirow{1}{*}{U-Net}
     & $\mathcal O(n)$ & 0.945$\pm$0.006 & 0.749$\pm$0.003 & 0.834$\pm$0.041 & 1.98$\pm$0.05 & 3.64$\pm$0.54 \\
    \multirow{1}{*}{DIVE} 
     & $\mathcal O(n)$ & \textbf{0.955$\pm$0.002} & 0.754$\pm$0.001 & 0.841$\pm$0.026 & 1.94$\pm$0.13 & 3.28$\pm$0.64\\
     \multirow{1}{*}{Mosin.}
     & $\mathcal O(n)$ & 0.954$\pm$0.005 & 0.722$\pm$0.001 & 0.887$\pm$0.039 & 1.17$\pm$0.03 & 2.78$\pm$0.29\\
     \multirow{1}{*}{TopoLoss}
     & $\mathcal O(n\log n)$ & 0.952$\pm$0.004 & 0.762$\pm$0.004 & 0.902$\pm$0.011 & 1.08$\pm$0.01 & 1.08$\pm$0.27 \\
    \multirow{1}{*}{DMT} 
     & $\mathcal O(n^2)$ & 0.955$\pm$0.004 & 0.773$\pm$0.004 & 0.902$\pm$0.002 & 0.88$\pm$0.04 & \textbf{0.87$\pm$0.40} \\
    \multirow{1}{*}{\textbf{Ours}} 
     & $\mathcal O(n)$ & 0.953$\pm$0.002 & \textbf{0.809$\pm$0.012} & \textbf{0.943$\pm$0.002} & \textbf{0.48$\pm$0.01} & 0.94$\pm$0.27\\
     \midrule
     \multicolumn{6}{c}{ISBI12} \\
     \midrule
     \multirow{1}{*}{U-Net}
     & $\mathcal O(n)$ & 0.968$\pm$0.002 & 0.970$\pm$0.005 & 0.934$\pm$0.007 & 1.37$\pm$0.03 & 2.79$\pm$0.27 \\
     \multirow{1}{*}{DIVE}
     & $\mathcal O(n)$ & 0.964$\pm$0.004 & 0.971$\pm$0.003 & 0.943$\pm$0.009 & 1.24$\pm$0.03 & 3.19$\pm$0.31 \\
     \multirow{1}{*}{Mosin.}
     & $\mathcal O(n)$ & 0.953$\pm$0.006 & 0.972$\pm$0.002 & 0.931$\pm$0.005 & 0.98$\pm$0.04 & 1.24$\pm$0.25 \\
     \multirow{1}{*}{TopoLoss}
     & $\mathcal O(n\log n)$ & 0.963$\pm$0.004 & 0.976$\pm$0.004 & 0.944$\pm$0.008 & 0.78$\pm$0.02 & 0.43$\pm$0.10 \\
     \multirow{1}{*}{DMT}
     & $\mathcal O(n^2)$ & 0.959$\pm$0.004 & 0.980$\pm$0.003 & \textbf{0.953$\pm$0.005} & \textbf{0.67$\pm$0.03} & \textbf{0.39$\pm$0.11} \\
     \multirow{1}{*}{\textbf{Ours}}
     & $\mathcal O(n)$ & \textbf{0.971$\pm$0.002} & \textbf{0.983$\pm$0.001} & 0.934$\pm$0.001 & 0.74$\pm$0.03 & 0.48$\pm$0.02 \\
    \midrule
     \multicolumn{6}{c}{CrackTree} \\
     \midrule
     \multirow{1}{*}{U-Net}
     & $\mathcal O(n)$ & 0.982$\pm$0.010 & 0.649$\pm$0.003 & 0.875$\pm$0.042 & 1.63$\pm$0.10 & 1.79$\pm$0.30 \\
     \multirow{1}{*}{DIVE}
     & $\mathcal O(n)$ & 0.985$\pm$0.005 & 0.653$\pm$0.002 & 0.863$\pm$0.0376 & 1.57$\pm$0.08 & 1.58$\pm$0.29 \\
     \multirow{1}{*}{Mosin.}
     & $\mathcal O(n)$ & 0.983$\pm$0.007 & 0.653$\pm$0.001 & 0.890$\pm$0.020 & 1.11$\pm$0.06 & 1.05$\pm$0.21 \\
     \multirow{1}{*}{TopoLoss}
     & $\mathcal O(n\log n)$ & 0.983$\pm$0.008 & 0.673$\pm$0.004 & 0.929$\pm$0.012 & 0.99$\pm$0.01 & 0.67$\pm$0.18 \\
     \multirow{1}{*}{DMT}
     & $\mathcal O(n^2)$ & 0.984$\pm$0.004 & \textbf{0.681$\pm$0.005} & \textbf{0.931$\pm$0.017} & \textbf{0.90$\pm$0.08} & 0.52$\pm$0.19 \\
     \multirow{1}{*}{\textbf{Ours}}
     & $\mathcal O(n)$ & \textbf{0.986$\pm$0.001} & 0.667$\pm$0.010 & 0.914$\pm$0.011 & 0.98$\pm$0.10 & \textbf{0.51$\pm$0.06} \\
    \midrule
    \multicolumn{6}{c}{EXASPIM} \\
    \midrule
     \multirow{1}{*}{U-Net} 
     & $\mathcal O(n)$ & 0.997$\pm$0.010 & 0.751$\pm$0.047  & 0.875$\pm$0.082 & 1.28$\pm$0.46 & 0.74$\pm$0.03 \\
     \multirow{1}{*}{Gornet} 
     & $\mathcal O(n^2)$ & 0.994$\pm$0.001 & 0.777$\pm$0.083 & 0.901$\pm$0.049 & 0.65$\pm$0.17 & 0.42$\pm$0.07\\
     \multirow{1}{*}{clDice} 
     & $\mathcal O(kn)$ & \textbf{0.998$\pm$0.004} & 0.785$\pm$0.032 & 0.923$\pm$0.071 & 0.66$\pm$0.51 & 0.36$\pm$0.07 \\
     \multirow{1}{*}{MALIS} 
     & $\mathcal O(n^2)$ & 0.997$\pm$0.001 & \textbf{0.794$\pm$0.052} & 0.927$\pm$0.042 & 0.64$\pm$0.27 & 0.34$\pm$0.08\\
     \multirow{1}{*}{\textbf{Ours}} 
     & $\mathcal O(n)$ & 0.997$\pm$0.001 & 0.770$\pm$0.058  & \textbf{0.953$\pm$0.038} & \textbf{0.42$\pm$0.21} & \textbf{0.31$\pm$0.06} \\
     \bottomrule
    \end{tabular}
    }
    \caption{{\bf Quantitative results for different models on several datasets.} Results for Dive, Mosin., TopoLoss, and DMT are compiled from~\cite{hu2023learn}. $\mathcal O(\cdot)$: complexity of training iterations, $n$: number of pixels/voxels, $k$: number of pooling operations in clDice.}
    \label{table:full-common-metrics}
\end{table*}

\begin{table*}[htbp!]
    \centering
    \resizebox{110mm}{!}{
    \begin{tabular}{llll}
    \midrule
    Method & Complexity \hspace{0.5mm} & \multirow{2}{*}{\shortstack{Runtime$\slash$Epoch \\ $64\times64\times64$}}  & \multirow{2}{*}{\shortstack{Runtime$\slash$Epoch \\ $128\times128\times128$}}   \\
    \\
    \midrule
    \multirow{1}{*}{U-Net} & $\mathcal O(n)$ & 2.59$\pm$0.18 &  10.03$\pm$0.23 sec \\
    \multirow{1}{*}{Gornet} & $\mathcal O(n^2)$ & 6.88$\pm$0.70 & 71.62$\pm$1.83 sec \\
    \multirow{1}{*}{clDice} & $\mathcal O(kn)$ & 6.31$\pm$0.71 & 48.55$\pm$1.60 sec \\
    \multirow{1}{*}{MALIS} & $\mathcal O(n^2)$ & 5.03$\pm$0.47 & 50.68$\pm$1.58 sec \\
    \multirow{1}{*}{\textbf{Ours}} & $\mathcal O(n)$ & 7.17$\pm$0.35 & 20.12$\pm$1.15 sec \\
    \bottomrule
    \end{tabular}
    }
    \caption{Average runtimes per epoch for various image patch sizes.}
    \label{table:full-runtimes}
\end{table*}

\begin{figure*}[htbp!]
    \centering
    \includegraphics[width=0.9\textwidth]{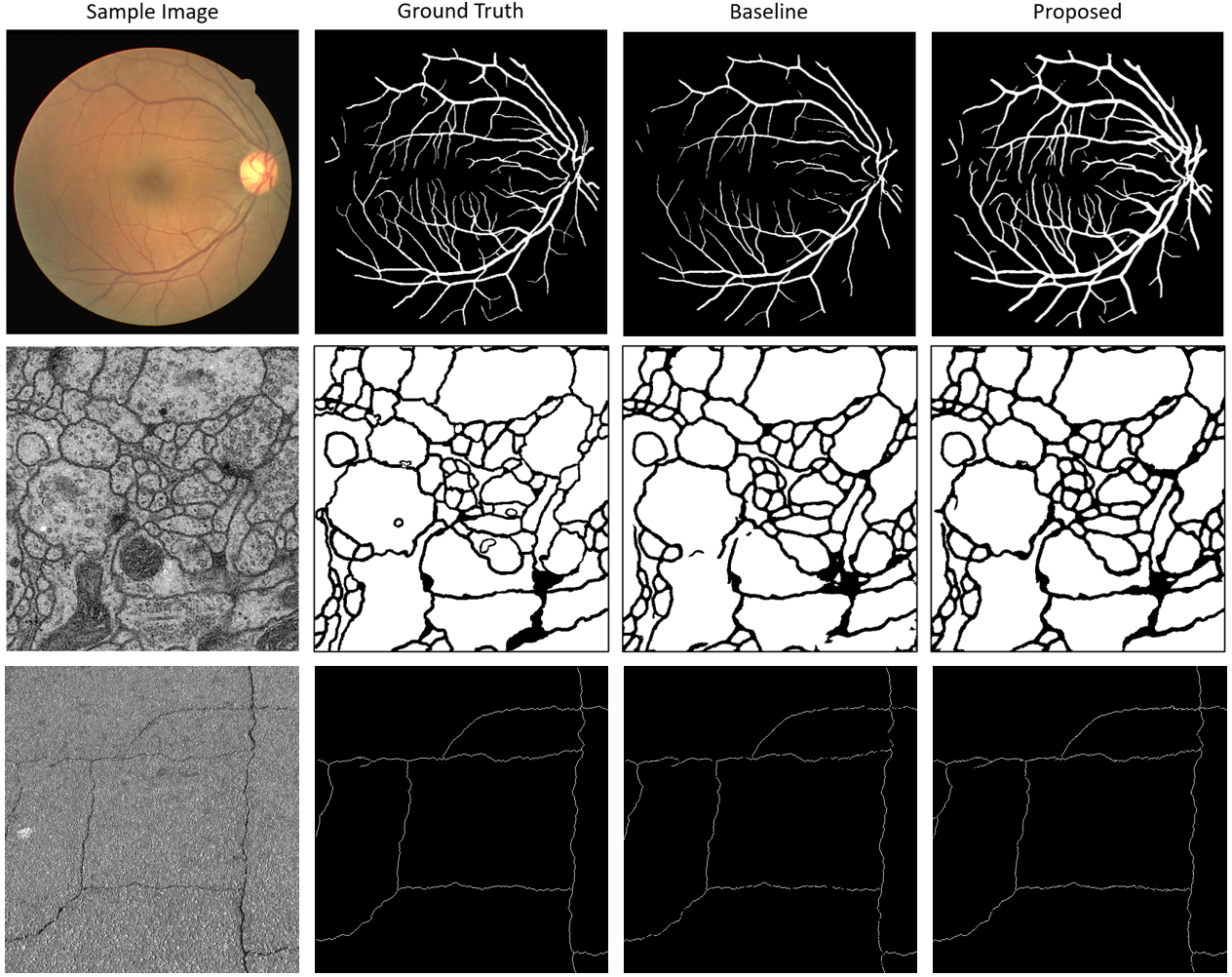}
    \caption{Qualitative results of different models on several 2-d segmentation datasets.}
    \label{fig:qualitative_results_2d_suppl}
\end{figure*}

\begin{figure*}[htbp!]
    \centering
    \includegraphics[width=0.9\textwidth]{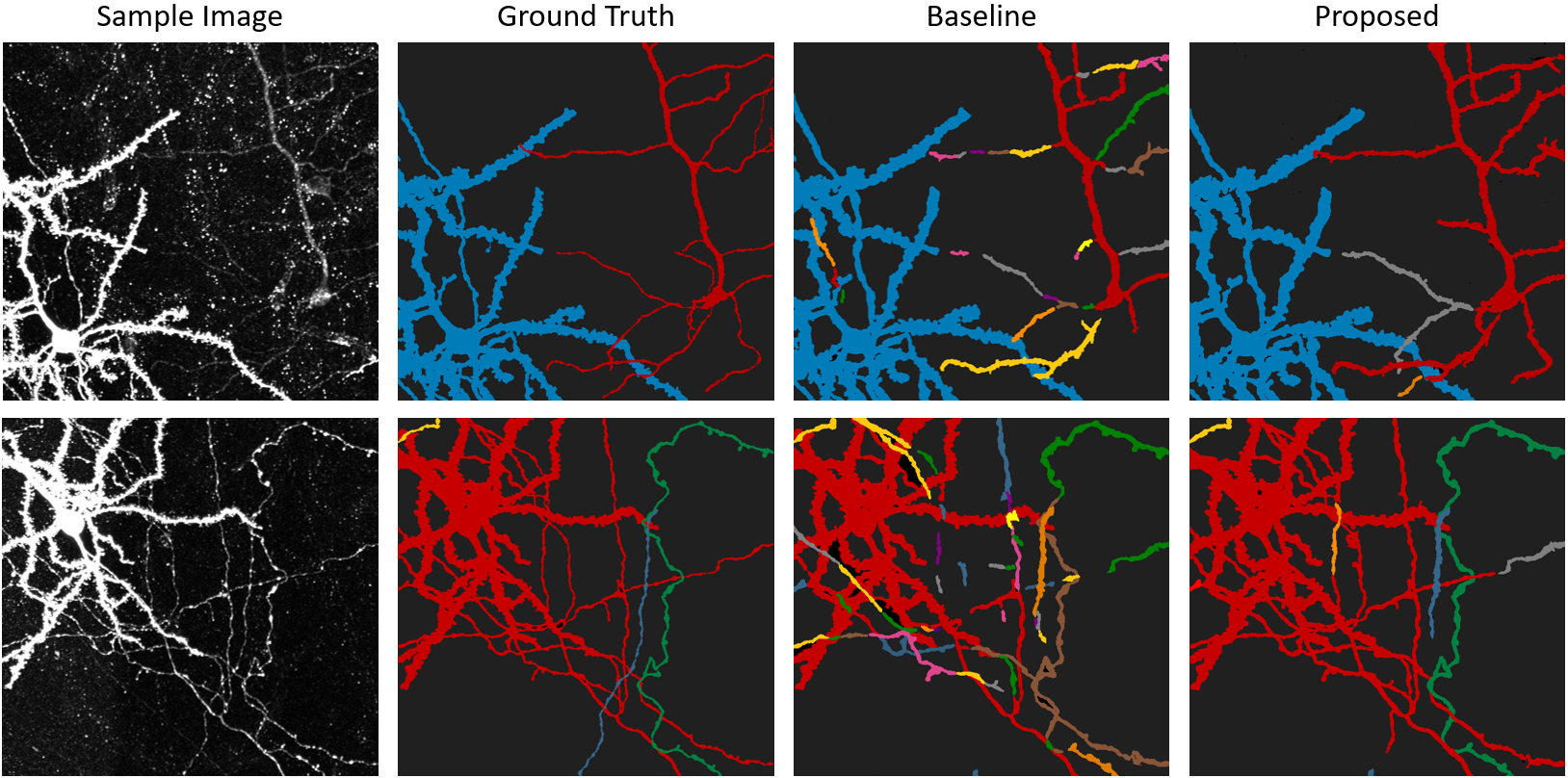}
    \caption{Qualitative results of different models on the 3-d EXASPIM dataset.}
    \label{fig:qualitative_results_3d_suppl}
\end{figure*}

\begin{figure*}[ht!]
    \centering
    \includegraphics[width=110mm]{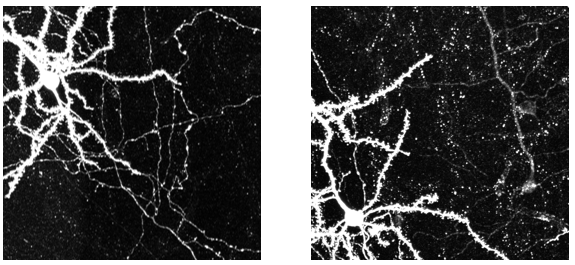}
    \caption{Raw image patches from the EXASPIM dataset that the results in Fig.~\ref{fig:qualitative_results_3d_main} were generated from.}
    \label{fig:raw_img}
\end{figure*}

\end{document}